\documentclass[12pt]{article}

\usepackage[margin=1.2in]{geometry}
\usepackage{amsmath, amsthm, amssymb, amsfonts}
\usepackage{mathtools}
\usepackage{bm}
\usepackage{hyperref}
\usepackage{cleveref}
\usepackage{booktabs}
\usepackage{array}
\usepackage{enumitem}
\usepackage{xcolor}
\usepackage{thmtools}
\usepackage{thm-restate}
\usepackage{natbib}
\usepackage{doi}
\usepackage{microtype}
\usepackage{graphicx}

\theoremstyle{plain}
\newtheorem{theorem}{Theorem}[section]
\newtheorem{lemma}[theorem]{Lemma}
\newtheorem{proposition}[theorem]{Proposition}
\newtheorem{corollary}[theorem]{Corollary}
\newtheorem{conjecture}[theorem]{Conjecture}
\theoremstyle{definition}
\newtheorem{definition}[theorem]{Definition}

\theoremstyle{remark}
\newtheorem{remark}[theorem]{Remark}

\hypersetup{
    colorlinks = true,
    linkcolor  = blue!70!black,
    citecolor  = green!50!black,
    urlcolor   = blue!70!black
}

\begin{document}

\title{
    \textbf{Fisher Widths: Local Learning Geometry and Anisotropic Recovery}}

\author{
Vu Khac Ky\\[4pt]
\small Department of Mathematics, FPT University, Vietnam\\
\small \texttt{kyvk2@fe.edu.vn}
}

\date{\today}
\maketitle

\begin{abstract}

We study Gaussian-width complexity on statistical manifolds through a pair of functionals: the primal Fisher width $w_G(T) = w(G^{1/2}T)$, induced by the Fisher metric, and the inverse-Fisher width $w_{G^{-1}}(T) = w(G^{-1/2}T)$, induced by the inverse Fisher metric. The two widths play complementary statistical roles.

On the learning side, the Fisher width measures the size of local parameter fluctuations in the geometry induced by the Fisher information. For Fisher-regular losses, we prove that the scale \(w_G(H_r)/\sqrt n\) is attained on sufficiently small Fisher balls.

On the recovery side, the inverse-Fisher width captures the effect of anisotropic Gaussian measurements whose covariance is determined by the inverse Fisher information. For sparse recovery, the resulting geometry depends not only on sparsity but also on the position of the active coordinates in the Fisher spectrum. We obtain a two-sided estimate for the corresponding statistical dimension, together with support-sensitive recovery estimates and a natural ordering of supports with different curvature profiles.

Finally, we establish a sharp relation between the primal and inverse-Fisher widths. On any common compact coordinate set $T$, they satisfy
\[
w_G(T)w_{G^{-1}}(T)\geq w(T)^2.
\]
Thus, Fisher anisotropy may transfer complexity from one geometry to the other, but cannot reduce both widths relative to the Euclidean scale.
\end{abstract}

\section{Introduction}
\label{sec:intro}
\subsection{From Fisher width to a primal--inverse pair}
\label{subsec:primal_inverse_intro}

The Fisher information matrix \(G(\theta)\) defines the local geometry of a statistical model. Directions of large Fisher curvature are directions in which the model distribution changes rapidly, while directions of small Fisher curvature are statistically flat. The Fisher metric therefore measures local sensitivity to parameter perturbations, whereas its inverse determines the covariance scale appearing in efficient-estimation geometry, as reflected in the Cram\'er--Rao bound. Thus \(G(\theta)\) and \(G(\theta)^{-1}\) induce two complementary deformations of local parameter sets.

Fisher width was introduced in \cite{ky2026a} as
\[
    w_G(T)=w\bigl(G(\theta)^{1/2}T\bigr),
\]
which extends classical Gaussian width to statistical models endowed with the Fisher metric. It is the Gaussian width of the Fisher-deformed set \(G(\theta)^{1/2}T\), and measures the size of a parameter set in the local Fisher geometry.

The present paper studies this width together with its inverse-metric counterpart,
\[
    w_{G^{-1}}(T)
    =
    w\bigl(G(\theta)^{-1/2}T\bigr).
\]
We refer to \(\bigl(w_G,w_{G^{-1}}\bigr)\) as a \emph{primal--inverse pair}, corresponding to the Fisher and inverse-Fisher deformations of the same local parameter set.

The two widths arise in different statistical settings. The Fisher width is associated with score fluctuations and local learning bounds, whereas the inverse-Fisher width appears in recovery problems with Gaussian measurement covariance \(G^{-1}\). When evaluated on a common localized coordinate set in a fixed chart, the two deformations respond oppositely to Fisher anisotropy. The main questions of this paper are how these widths enter learning and recovery bounds, and what relations constrain them when they are applied to the same coordinate set.

Throughout this paper, the Fisher information matrix \(G(\theta)\) is evaluated at a fixed reference point \(\theta_0\), and all width functionals are computed in a chosen local parameter chart. The corresponding tangent and cotangent transformation laws are recorded in Section~\ref{subsec:coordinate_transformations}.

The Fisher information matrix \(G(\theta)\) defines the local geometry of a statistical model. Directions of large Fisher curvature are directions in which the model distribution changes rapidly, while directions of small Fisher curvature are statistically flat. The Fisher metric therefore measures local sensitivity to parameter perturbations, whereas its inverse describes the corresponding scale of estimation uncertainty, as reflected in the Cram\'er--Rao bound. Thus \(G(\theta)\) and \(G(\theta)^{-1}\) induce two complementary deformations of local parameter sets.

Fisher width was introduced in \cite{ky2026a} as
\[
    w_G(T)=w\bigl(G(\theta)^{1/2}T\bigr),
\]
which extends classical Gaussian width to statistical models endowed with the Fisher metric. It is the Gaussian width of the Fisher-deformed set \(G(\theta)^{1/2}T\), and measures the size of a parameter set in the local Fisher geometry.

The present paper studies this width together with its inverse-metric counterpart,
\[
    w_{G^{-1}}(T)
    =
    w\bigl(G(\theta)^{-1/2}T\bigr).
\]
We refer to \(\bigl(w_G,w_{G^{-1}}\bigr)\) as a \emph{primal--inverse pair}, corresponding to the Fisher and inverse-Fisher deformations of the same local parameter set.

The two widths arise in different statistical settings. The Fisher width is associated with score fluctuations and learning complexity, whereas the inverse-Fisher width appears in recovery problems with Gaussian measurement covariance \(G^{-1}\). When evaluated on a common localized coordinate set, the two deformations respond oppositely to Fisher anisotropy. The main questions of this paper are how these widths enter learning and recovery bounds, and what relations constrain them when the same local geometry is relevant to both problems.

Throughout this paper, the Fisher information matrix \(G(\theta)\) is evaluated at a fixed reference point
\(\theta_0\), and all width functionals are computed in a chosen local parameter chart. The corresponding tangent and cotangent transformation laws are recorded in Section~\ref{subsec:coordinate_transformations}.

\subsection{Related work}
\label{subsec:related}

\noindent\textbf{Gaussian width and high-dimensional geometry.}
Gaussian width is a central complexity measure in asymptotic convex geometry, high-dimensional probability, and empirical process theory \citep{talagrand2005,ledoux1991probability,vershynin2018,wainwright2019}. It measures the size of a set through its interaction with a Gaussian process and appears in concentration, random projection, embedding, and
uniform-deviation estimates \citep{boucheron2013concentration,plan2014dimension}. The Fisher width introduced in \citet{ky2026a} may be viewed as a Fisher-geometric analogue of Gaussian width, obtained by deforming the parameter set through the local Fisher metric.

\medskip
\noindent\textbf{Conic phase transitions and convex recovery.}
The geometric theory of recovery thresholds begins with Gordon's escape-through-a-mesh theorem \citep{gordon1988}. Compressed sensing and convex recovery subsequently connected exact recovery with descent cones, Gaussian width, and statistical dimension \citep{candes2006robust,donoho2006compressed,donoho2009, foucart2013mathematical,chandrasekaran2012,amelunxen2014}.
For Gaussian measurements, Gordon's theorem gives a sufficient condition in terms of Gaussian width, whereas the statistical dimension determines the sharp conic transition \citep{amelunxen2014}. We use this framework through the identity
\[
    G^{-1/2}D(\|\cdot\|_1,x^\star)
    =
    D\bigl(\|G^{1/2}\cdot\|_1,G^{-1/2}x^\star\bigr),
\]
which reduces inverse-Fisher recovery to a standard Gaussian recovery problem with a weighted \(\ell_1\) descent cone. The corresponding upper functional is the standard weighted-\(\ell_1\)
distance-to-subdifferential expression. Our contribution is its Fisher interpretation and a two-sided estimate obtained by optimizing the standard ALMT error term over vectors with the same support and sign pattern, together with the resulting support-ordering consequences.

\medskip
\noindent\textbf{Anisotropic random measurements and weighted \(\ell_1\) recovery.}
Classical Gaussian recovery theory assumes isotropic measurement ensembles. \citet{KuengGross2014} established RIPless compressed-sensing bounds for anisotropic ensembles with sampling rates depending on the condition number of the covariance matrix, while \citet{RudelsonZhou2013} analyzed sparse recovery and restricted eigenvalue conditions for subgaussian matrices with nontrivial covariance.

Weighted \(\ell_1\) minimization has also been studied for nonuniform sparsity, partial support information, and known support distributions \citep{khajehnejad2011nonuniform,diaz2018known}. In that literature, weights are typically chosen from prior information about the support and may be optimized to improve the recovery threshold. Regularization and descent-cone geometry more generally provide a standard framework for structured recovery \citep{tibshirani1996lasso,chandrasekaran2012,negahban2012unified,amelunxen2014}.

Our sensing model belongs to the anisotropic family, but its covariance is generated by an underlying statistical model. In a Gaussian location experiment, paired differences produce Gaussian sensing rows with covariance \(G^{-1}\), yielding measurement operators of the form \(AG^{-1/2}\). After whitening, the original unweighted \(\ell_1\) regularizer becomes
\[
    f_G(x)=\|G^{1/2}x\|_1.
\]
For diagonal \(G\), the resulting weights are therefore determined by the measurement covariance rather than chosen from prior support information. Accordingly, \(U_G(S)\) is not new in functional form; it is the standard weighted-\(\ell_1\) expression specialized to Fisher-induced weights. The new points are the statistical origin of these weights, optimization of the standard ALMT error bound over fixed support and sign patterns, and the resulting support-sensitive recovery estimate.

\medskip
\noindent\textbf{Information geometry and Fisher metrics.}
The Fisher information matrix defines the canonical Riemannian metric on statistical manifolds \citep{rao1945,amari2000,cencov1982statistical}. Its inverse appears in classical estimation theory through the Cram\'er--Rao bound and also underlies natural-gradient methods \citep{amari1998natural,pascanu2013revisiting}. Existing work has focused mainly on estimation, divergence geometry, statistical efficiency, and natural-gradient optimization. In contrast, we study Gaussian-width complexity under the two linear deformations induced by \(G\) and \(G^{-1}\).

\medskip
\noindent\textbf{Fisher information in machine learning.}
Fisher information is also used as a curvature matrix in machine learning, notably in natural-gradient methods and approximations such as K-FAC \citep{martens2015kfac}. Empirical Fisher approximations may differ from the population Fisher or Hessian \citep{kunstner2019}; our analysis uses the Fisher matrix to define geometric complexity rather than as an optimization preconditioner.

\subsection{Main contributions}
\label{subsec:contributions}

Our main contributions are as follows.

\begin{enumerate}

\item \textbf{Local attainment on Fisher balls.}
For Fisher-regular losses, we prove a finite-sample lower bound of order
\[
    \frac{w_G(H_r)}{\sqrt{n}},
\]
attaining the same scale as the standard Fisher-Lipschitz upper bound on sufficiently small Fisher balls. The result applies to standard correctly specified models under local
leverage and fourth-moment assumptions.

\item \textbf{Support-sensitive anisotropic recovery.}
For diagonal \(G\) and unweighted basis pursuit, the transformed cone is the descent cone of the weighted norm
\[
    x\longmapsto \|G^{1/2}x\|_1.
\]
The corresponding standard weighted-\(\ell_1\) upper functional is \(U_G(S)\). By optimizing the ALMT error bound over vectors with the same support and sign pattern, we obtain
\[
    U_G(S)
    -
    2\sqrt{\frac{\operatorname{Tr}(G)}
    {\sum_{i\in S}\gamma_i}}
    \le
    \delta\!\left(G^{-1/2}D(\|\cdot\|_1,x^\star)\right)
    \le
    U_G(S).
\]
The same representation yields monotonicity under nested supports and shows that replacing active coordinates by coordinates of larger Fisher curvature increases the upper recovery estimate.

\item \textbf{Sharp primal--inverse width inequality.}
For every nonempty compact set \(T\subset\mathbb R^d\) and every \(G\succ0\), we prove
\[
    w_G(T)\,w_{G^{-1}}(T)\ge w(T)^2.
\]
The constant is sharp, and the proof follows from log-convexity along commuting powers of the metric. We also derive related bounds for more general pairs of positive-definite matrices.

\item \textbf{Fisher interpretation of inverse-covariance recovery.}
Differences of Gaussian location observations generate sensing rows with covariance \(G^{-1}\). After whitening, convex recovery is governed by the transformed cone
\[
    G^{-1/2}D(R,x^\star).
\]
Gordon's theorem gives a sufficient recovery scale through Gaussian width, while statistical dimension locates the corresponding sharp Gaussian conic transition.

\end{enumerate}

\section{Fisher and Inverse-Fisher Widths}
\label{sec:fisher_widths}

\subsection{Definitions and basic properties}
\label{subsec:defs_basic}

Let \(\{p_\theta:\theta\in\Theta\subset\mathbb R^d\}\) be a parametric family of probability densities with respect to a base measure \(\mu\). The Fisher information matrix at \(\theta\) is
\begin{equation}
    G(\theta)_{ij} = 
    \mathbb E_{p_\theta}
    \left[
        \partial_i\log p_\theta(X)\,
        \partial_j\log p_\theta(X)
    \right].
    \label{eq:fisher_matrix}
\end{equation}
Throughout the paper, we evaluate the Fisher matrix at a fixed reference point \(\theta_0\). We assume that \(G:=G(\theta_0)\succ0,\) and work in a chosen local parameter chart around \(\theta_0\). The Fisher metric and its inverse induce the norm pair
\[
    \|v\|_G=(v^\top Gv)^{1/2},
    \qquad
    \|s\|_{G^{-1}}=(s^\top G^{-1}s)^{1/2}.
\]
For a compact set \(S\subset\mathbb R^d\), its Gaussian width is defined by
\[
    w(S)
    :=
    \mathbb E_g\sup_{v\in S}\langle g,v\rangle,
    \qquad
    g\sim N(0,I_d).
\]

\begin{definition}[Fisher and inverse-Fisher widths]
\label{def:fisher_widths}
Let \(G\succ0\) and let \(T\subset\mathbb R^d\) be compact. The \emph{Fisher width} and \emph{inverse-Fisher width} of \(T\) are
\begin{align}
    w_G(T)
    &:=
    w(G^{1/2}T)
    =
    \mathbb E_g\sup_{v\in T}
    \langle g,G^{1/2}v\rangle,
    \label{eq:primal_width}
    \\[3pt]
    w_{G^{-1}}(T)
    &:=
    w(G^{-1/2}T)
    =
    \mathbb E_g\sup_{v\in T}
    \langle g,G^{-1/2}v\rangle.
    \label{eq:inverse_width}
\end{align}
\end{definition}

The first width is induced by the Fisher metric \(G\), whereas the second is induced by the inverse metric \(G^{-1}\). We call \((w_G,w_{G^{-1}})\) the \emph{primal--inverse pair}; this terminology refers only to the two metric deformations. When \(G=I_d\), both quantities reduce to the classical Gaussian width.

\begin{lemma}[Basic properties]
\label{lem:basic_properties}
Let \(G\succ0\) and let \(T\subset\mathbb R^d\) be compact. Then:
\begin{enumerate}[label=(\roman*)]

\item \textup{(Fisher spectral bounds)}
\[
    \sqrt{\lambda_{\min}(G)}\,w(T)
    \le
    w_G(T)
    \le
    \sqrt{\lambda_{\max}(G)}\,w(T).
\]

\item \textup{(Inverse-Fisher spectral bounds)}
\[
    \frac{w(T)}{\sqrt{\lambda_{\max}(G)}}
    \le
    w_{G^{-1}}(T)
    \le
    \frac{w(T)}{\sqrt{\lambda_{\min}(G)}}.
\]

\item \textup{(Perturbation stability)}
If \(G_1,G_2\succ0\), then
\[
    |w_{G_1}(T)-w_{G_2}(T)|
    \le
    \|G_1^{1/2}-G_2^{1/2}\|_{\mathrm{op}}\,w(T),
\]
and
\[
    |w_{G_1^{-1}}(T)-w_{G_2^{-1}}(T)|
    \le
    \|G_1^{-1/2}-G_2^{-1/2}\|_{\mathrm{op}}\,w(T).
\]

\end{enumerate}
\end{lemma}

\begin{proof}
For the upper bound in~\textup{(i)}, consider the centered Gaussian processes
\[
    X_v:=\langle g,G^{1/2}v\rangle,
    \qquad
    Y_v:=\sqrt{\lambda_{\max}(G)}\,\langle g,v\rangle,
    \qquad v\in T.
\]
For every \(u,v\in T\),
\[
    \mathbb E|X_v-X_u|^2
    =
    (v-u)^\top G(v-u)
    \le
    \lambda_{\max}(G)\|v-u\|_2^2
    =
    \mathbb E|Y_v-Y_u|^2.
\]
The Sudakov--Fernique comparison theorem therefore gives
\[
    w_G(T)
    =
    \mathbb E\sup_{v\in T}X_v
    \le
    \mathbb E\sup_{v\in T}Y_v
    =
    \sqrt{\lambda_{\max}(G)}\,w(T).
\]
Comparing instead \(\sqrt{\lambda_{\min}(G)}\,\langle g,v\rangle\) with \(\langle g,G^{1/2}v\rangle\) gives the lower bound. Applying the same argument to \(G^{-1}\) proves~\textup{(ii)}. Part~\textup{(iii)} is \citet[Theorem~3.1]{ky2026a} applied first to \(G_1,G_2\) and then to
\(G_1^{-1},G_2^{-1}\).
\end{proof}

In applications, the population Fisher matrix \(G\) may be replaced by an empirical estimate \(\widehat G\). Lemma~\ref{lem:basic_properties} then controls the induced errors in both widths. Since the inverse square-root map is unstable near singular matrices, estimating the inverse-Fisher width requires a uniform positive lower bound on the relevant Fisher eigenvalues.

\subsection{Statistical interpretation}
\label{subsec:statistical_interpretation}

The inverse Fisher metric is dual to the Fisher norm: for every covector \(s\in\mathbb R^d\),
\[
    \sup_{\|h\|_G\le1}\langle s,h\rangle
    =
    \|s\|_{G^{-1}}.
\]
Indeed,
\[
    \langle s,h\rangle
    =
    \langle G^{-1/2}s,G^{1/2}h\rangle
    \le
    \|s\|_{G^{-1}}\|h\|_G
\]
by Cauchy--Schwarz, with equality for
\(
    h
    =
    \frac{G^{-1}s}{\|s\|_{G^{-1}}}
\)
when \(s\neq0\). Thus score vectors and loss gradients, which act linearly on parameter perturbations, are naturally measured in the inverse Fisher norm.

The two widths also admit Gaussian-process representations. If \(S\sim N(0,G)\) and \(\Delta\sim N(0,G^{-1})\), then
\begin{equation}
    w_G(T)
    =
    \mathbb E\sup_{h\in T}\langle S,h\rangle,
    \qquad
    w_{G^{-1}}(T)
    =
    \mathbb E\sup_{h\in T}\langle \Delta,h\rangle.
    \label{eq:gaussian_pairing}
\end{equation}
Under the usual regularity conditions, the score at \(\theta_0\) is centered with covariance \(G\), and its normalized sum converges to \(N(0,G)\). Likewise, an asymptotically efficient estimator satisfies
\[
    \sqrt n(\widehat\theta_n-\theta_0)
    \Rightarrow
    N(0,G^{-1}).
\]
Hence the first process in~\eqref{eq:gaussian_pairing} has the covariance geometry of local score fluctuations, whereas the second has the covariance geometry of efficient estimation errors. Section~\ref{sec:primal_learning_complexity} develops the learning-side role of \(w_G\), while Section~\ref{sec:fisher_escape} studies recovery under inverse-Fisher Gaussian measurements.

\subsection{Coordinate transformations}
\label{subsec:coordinate_transformations}

We record the transformation laws that distinguish intrinsic statements from comparisons made on a common coordinate set. Let \(\theta'=\varphi(\theta)\) be a smooth reparametrization with invertible
Jacobian \(J:=D\varphi(\theta_0)\). Tangent vectors, covectors, and the Fisher matrix transform as
\[
    h'=Jh,
    \qquad
    s'=J^{-\top}s,
    \qquad
    G'=J^{-\top}GJ^{-1}.
\]
Throughout this subsection, all matrix square roots are the principal symmetric square roots.

\begin{proposition}[Tangent and cotangent transformation laws]
\label{prop:reparam_invariance}
Let \(T\subset\mathbb R^d\) be a compact set of tangent vectors and let \(S\subset\mathbb R^d\) be a compact set of covectors. Then
\[
    w_{G'}(JT)=w_G(T),
    \qquad
    w_{(G')^{-1}}(J^{-\top}S)=w_{G^{-1}}(S).
\]
\end{proposition}

\begin{proof}
Set
\[
    Q:=(G')^{1/2}JG^{-1/2}
      =(J^{-\top}GJ^{-1})^{1/2}JG^{-1/2}.
\]
Then
\[
    Q^\top Q
    =
    G^{-1/2}J^\top G'JG^{-1/2}
    =
    G^{-1/2}J^\top
    (J^{-\top}GJ^{-1})
    JG^{-1/2}
    =
    I_d.
\]
Thus \(Q\) is orthogonal and
\[
    (G')^{1/2}J=QG^{1/2}.
\]
By rotational invariance of the standard Gaussian law,
\[
    w_{G'}(JT)
    =
    w\bigl((G')^{1/2}JT\bigr)
    =
    w(QG^{1/2}T)
    =
    w(G^{1/2}T)
    =
    w_G(T).
\]

For the cotangent identity, note that
\[
    (G')^{-1}=JG^{-1}J^\top.
\]
Applying the same argument to the metric \(G^{-1}\) under the coordinate map \(s\mapsto J^{-\top}s\) gives
\[
    w_{(G')^{-1}}(J^{-\top}S)=w_{G^{-1}}(S).
\]
\end{proof}

\begin{remark}[Fixed-chart comparisons]
\label{rem:fixed_chart}
Proposition~\ref{prop:reparam_invariance} concerns tangent and cotangent sets transformed according to their respective intrinsic laws. Consequently, \(w_G(T)\) and \(w_{G^{-1}}(T)\), when evaluated on the same coordinate subset, do not form a jointly invariant pair under an arbitrary reparametrization. All same-set comparisons below are therefore understood in the fixed local chart chosen at \(\theta_0\). Regular exponential families provide a canonical dual-coordinate interpretation; see Appendix~\ref{app:expfam}.
\end{remark}

\section{A Local Lower Bound on Fisher Balls}
\label{sec:primal_learning_complexity}

For a loss class
\(
    \mathcal F_T=\{\ell_\theta:\theta\in T\},
\)
generalization asks how well the empirical risk \(\widehat R_n(\theta)\) approximates the population risk \(R(\theta)\), uniformly over \(T\). When the loss is Fisher-Lipschitz with constant
\(L\), meaning
\[
    |\ell(\theta;z)-\ell(\theta';z)|
    \le
    L\|\theta-\theta'\|_G
    \qquad
    \forall\,\theta,\theta'\in T,\ \forall\,z,
\]
standard symmetrization and contraction give, with probability at least \(1-\delta\),
\[
    \sup_{\theta\in T}
    |R(\theta)-\widehat R_n(\theta)|
    \le
    CL\frac{w_G(T)}{\sqrt n}
    +
    B\sqrt{\frac{\log(1/\delta)}{2n}},
\]
where \(B\) bounds the loss and we have used \(w_G(T-T)\le 2w_G(T)\) for convex symmetric \(T\) containing the origin; see \cite{ky2026a} for the full derivation. Thus \(w_G(T)/\sqrt n\) provides the standard Fisher-geometric upper scale for uniform empirical fluctuations. Analogous bounds hold under suitable concentration assumptions in place of boundedness.

This section gives a non-asymptotic lower bound on Fisher balls, showing that the order \(w_G(H_r)/\sqrt n\) is attained for a class of Fisher-regular losses on sufficiently small local neighborhoods. Since \(G^{1/2}H_r=rB_2^d\), the result establishes the local dimensional scale \(r\sqrt{d/n}\). It does not provide a lower-bound principle for arbitrary structured sets or a minimax characterization in terms of Fisher width.

\subsection{A local lower bound on Fisher balls}
\label{subsec:fisher-regular-lower}

In a regular exponential family, the deterministic log-partition term cancels from the centered empirical fluctuation, leaving a linear score process. The example in \cite{ky2026a} therefore satisfies
\[
    \sqrt n\,\mathbb E\sup_{u\in rB_2^d}
    |R(u)-\widehat R_n(u)|
    \longrightarrow
    w_{G_0}(rB_2^d)
    \qquad\text{as }n\to\infty.
\]
The result below replaces this model-specific asymptotic argument by a finite-sample linearization with a controlled quadratic remainder on the Fisher ball \(H_r:=\{h:\|h\|_G\le r\}\).

Throughout this subsection, \(Z,Z_1,\ldots,Z_n\stackrel{\mathrm{iid}}{\sim}P\), and we write
\[
    Pf:=\mathbb E[f(Z)],
    \qquad
    P_nf:=\frac1n\sum_{i=1}^n f(Z_i).
\]
All expectations and covariances are taken under \(P\), unless stated otherwise.

\begin{definition}[Fisher-regular loss at \(\theta_0\)]
\label{def:fisher-regular}
Let \(\rho>0\). A loss \(\ell:\Theta\times\mathcal Z\to\mathbb R\), twice differentiable on \(\{\theta\in\Theta:\|\theta-\theta_0\|_G\le\rho\}\), is called \emph{Fisher-regular at \(\theta_0\) with radius \(\rho\)} and constants \((G,\kappa,\sigma_H)\) if the following conditions hold.

\begin{enumerate}[label=\textup{(FR\arabic*)}]

\item\label{fr:score}
\textup{(Centered gradient and Fisher covariance)}
\[
    \mathbb E\!\left[\nabla_\theta\ell_{\theta_0}(Z)\right]=0,
    \qquad
    \operatorname{Cov}\!\left(
        \nabla_\theta\ell_{\theta_0}(Z)
    \right)=G\succ0.
\]

\item\label{fr:hessian}
\textup{(Local Hessian \(L^2\)-bound)}
There exists a random variable \(M(Z)\ge0\) such that
\[
    \bigl(\mathbb E M(Z)^2\bigr)^{1/2}\le\sigma_H
\]
and, almost surely, for every \(\theta\) satisfying \(\|\theta-\theta_0\|_G\le\rho\) and every \(h\in\mathbb R^d\),
\[
    \bigl|h^\top\nabla_\theta^2\ell_\theta(Z)h\bigr|
    \le
    M(Z)\|h\|_G^2.
\]

\item\label{fr:nondegen}
\textup{(Whitened gradient fourth-moment bound)}
The whitened gradient
\(
    \zeta:=G^{-1/2}\nabla_\theta\ell_{\theta_0}(Z)
\)
satisfies
\[
    \mathbb E\|\zeta\|_2^4\le\kappa^4d^2.
\]

\end{enumerate}
\end{definition}

Under Condition~\ref{fr:score}, the whitened gradient \(\zeta\) is centered and isotropic. Condition~\ref{fr:nondegen} supplies the fourth-moment control used below, while Condition~\ref{fr:hessian} controls the Taylor remainder.

\begin{lemma}[Remainder bounds]
\label{lem:remainder}

Let \(\ell\) be Fisher-regular at \(\theta_0\) with radius \(\rho\) and
constants \((G,\kappa,\sigma_H)\). For \(0<r\le\rho\), define
\[
    R_h(Z)
    :=
    \ell_{\theta_0+h}(Z)
    -
    \ell_{\theta_0}(Z)
    -
    \left\langle
        \nabla_\theta\ell_{\theta_0}(Z),h
    \right\rangle.
\]
Then, for all \(h,h'\in H_r\), the following bounds hold almost surely:

\begin{enumerate}[label=\textup{(\roman*)}]

\item\label{rb:ptwise}
\(
    |R_h(Z)|
    \le
    \frac12 M(Z)\|h\|_G^2.
\)

\item\label{rb:l2}
\(
    \bigl(\mathbb E R_h(Z)^2\bigr)^{1/2}
    \le
    \frac12\sigma_H\|h\|_G^2.
\)

\item\label{rb:lip}
\(
    |R_h(Z)-R_{h'}(Z)|
    \le
    M(Z)\,r\,\|h-h'\|_G.
\)

\end{enumerate}
\end{lemma}

\begin{proof}
Taylor's theorem with integral remainder gives
\[
    R_h(Z)
    =
    \int_0^1(1-t)\,
    h^\top\nabla_\theta^2\ell_{\theta_0+th}(Z)h\,dt.
\]
Parts~\ref{rb:ptwise} and~\ref{rb:l2} follow immediately from Condition~\ref{fr:hessian}.

For part~\ref{rb:lip}, set \(u(s):=h'+s(h-h')\). Since \(H_r\) is convex, \(\|u(s)\|_G\le r\) for \(s\in[0,1]\). Two applications of the fundamental theorem of calculus give
\[
    R_h(Z)-R_{h'}(Z)
    =
    \int_0^1\int_0^1
    (h-h')^\top
    \nabla_\theta^2\ell_{\theta_0+t u(s)}(Z)
    u(s)\,dt\,ds.
\]
Because the Hessian is symmetric, Condition~\ref{fr:hessian} is equivalent to
\[
    \left\|
        G^{-1/2}
        \nabla_\theta^2\ell_\theta(Z)
        G^{-1/2}
    \right\|_{\mathrm{op}}
    \le
    M(Z),
\]
and therefore implies
\[
    \bigl|
        a^\top\nabla_\theta^2\ell_\theta(Z)b
    \bigr|
    \le
    M(Z)\|a\|_G\|b\|_G
    \qquad
    \text{for all }a,b\in\mathbb R^d.
\]
Applying this bound inside the double integral and using \(\|u(s)\|_G\le r\) proves~\ref{rb:lip}.
\end{proof}

\begin{lemma}[Remainder empirical complexity]
\label{lem:rem-complexity}

Under the assumptions of Lemma~\ref{lem:remainder},
\[
    \mathbb E\sup_{h\in H_r}
    \bigl|(P_n-P)R_h\bigr|
    \le
    C\,\sigma_H\,r\,
    \frac{w_G(H_r)}{\sqrt n},
\]
where \(C>0\) is a universal constant.
\end{lemma}

\begin{proof}
By symmetrization,
\[
    \mathbb E\sup_{h\in H_r}
    \bigl|(P_n-P)R_h\bigr|
    \le
    2\,
    \mathbb E\sup_{h\in H_r}|X_h|,
    \qquad
    X_h:=
    \frac1n\sum_{i=1}^n\varepsilon_iR_h(Z_i),
\]
where \(\varepsilon_1,\ldots,\varepsilon_n\) are independent Rademacher variables, independent of the sample.

Conditionally on \(Z_1,\ldots,Z_n\), the process \((X_h)_{h\in H_r}\) is symmetric, satisfies \(X_0=0\), and has sub-Gaussian increments with respect to
\[
    \widetilde d_n(h,h')
    :=
    \left(
        \mathbb E_\varepsilon
        |X_h-X_{h'}|^2
    \right)^{1/2}.
\]
By Lemma~\ref{lem:remainder}\ref{rb:lip},
\[
    \widetilde d_n(h,h')
    \le
    \frac{\overline M_n r}{\sqrt n}\,
    \|h-h'\|_G,
    \qquad
    \overline M_n
    :=
    \left(
        \frac1n\sum_{i=1}^nM(Z_i)^2
    \right)^{1/2}.
\]
Since the conditional process is symmetric and \(X_0=0\), the generic chaining upper bound for sub-Gaussian processes yields
\[
    \mathbb E_\varepsilon
    \sup_{h\in H_r}|X_h|
    \le
    C\gamma_2(H_r,\widetilde d_n)
    \le
    C\frac{\overline M_n r}{\sqrt n}
    \gamma_2(H_r,\|\cdot\|_G).
\]

Since \(h\mapsto G^{1/2}h\) is an isometry from \((H_r,\|\cdot\|_G)\) onto \((rB_2^d,\|\cdot\|_2)\),
\[
    \gamma_2(H_r,\|\cdot\|_G)
    =
    \gamma_2(rB_2^d,\|\cdot\|_2)
    \asymp
    r\sqrt d
    \asymp
    w_G(H_r);
\]
see \citet{talagrand2005}. Therefore,
\[
    \mathbb E_\varepsilon
    \sup_{h\in H_r}|X_h|
    \le
    C\overline M_n r\,
    \frac{w_G(H_r)}{\sqrt n}.
\]
Finally, Jensen's inequality and Condition~\ref{fr:hessian} give
\[
    \mathbb E\overline M_n
    \le
    \bigl(\mathbb E\overline M_n^2\bigr)^{1/2}
    =
    \bigl(\mathbb E M(Z)^2\bigr)^{1/2}
    \le
    \sigma_H.
\]
Taking expectation over the sample completes the proof.
\end{proof}

\begin{theorem}[Local Fisher-width lower bound for Fisher-regular losses]
\label{thm:fisher-regular-lower}

Let \(\ell\) be Fisher-regular at \(\theta_0\) with radius \(\rho\) and constants \((G,\kappa,\sigma_H)\). Let \(0<r\le\rho\), and suppose
\[
    C\sigma_Hr\le\frac{c_\kappa}{2},
    \qquad
    c_\kappa:=\frac{1}{16\sqrt2\,\kappa^4},
\]
where \(C\) is the universal constant in Lemma~\ref{lem:rem-complexity}. Then
\[
    \mathbb E\sup_{h\in H_r}
    \left|
        (P_n-P)
        \bigl(
            \ell_{\theta_0+h}-\ell_{\theta_0}
        \bigr)
    \right|
    \ge
    \frac{c_\kappa}{2}
    \frac{w_G(H_r)}{\sqrt n}.
\]
\end{theorem}

\begin{proof}
Writing
\[
    \ell_{\theta_0+h}-\ell_{\theta_0}
    =
    \left\langle
        \nabla_\theta\ell_{\theta_0},h
    \right\rangle
    +
    R_h,
\]
the left-hand side is bounded below by \(A-B\), where
\[
    A
    :=
    \mathbb E\sup_{h\in H_r}
    \left|
        (P_n-P)
        \left\langle
            \nabla_\theta\ell_{\theta_0},h
        \right\rangle
    \right|,
    \qquad
    B
    :=
    \mathbb E\sup_{h\in H_r}
    |(P_n-P)R_h|.
\]

Define
\[
    \zeta_i
    :=
    G^{-1/2}\nabla_\theta\ell_{\theta_0}(Z_i),
    \qquad
    S_n
    :=
    \frac1{\sqrt n}\sum_{i=1}^n\zeta_i.
\]
By Condition~\ref{fr:score}, the variables \(\zeta_i\) are i.i.d., centered, and isotropic. Since
\[
    \frac1n\sum_{i=1}^n
    \nabla_\theta\ell_{\theta_0}(Z_i)
    =
    \frac1{\sqrt n}G^{1/2}S_n
\]
and \(G^{1/2}H_r=rB_2^d\), we obtain
\[
    A
    =
    \frac r{\sqrt n}\,
    \mathbb E\|S_n\|_2.
\]

Isotropy gives \(\mathbb E\|S_n\|_2^2=d\). Moreover,
\[
\begin{aligned}
    \mathbb E
    \left\|
        \sum_{i=1}^n\zeta_i
    \right\|_2^4
    &=
    n\,\mathbb E\|\zeta\|_2^4
    +
    n(n-1)d^2
    +
    2n(n-1)d.
\end{aligned}
\]
Therefore,
\[
\begin{aligned}
    \mathbb E\|S_n\|_2^4
    &=
    \frac1n\mathbb E\|\zeta\|_2^4
    +
    \left(1-\frac1n\right)(d^2+2d) 
    \le
    4\kappa^4d^2,
\end{aligned}
\]
where the last inequality uses Condition~\ref{fr:nondegen}, \(\kappa^4\ge1\), and \(d\ge1\).

Applying the Paley--Zygmund inequality to \(Y=\|S_n\|_2^2\) yields
\[
    \mathbb P
    \left(
        \|S_n\|_2^2\ge\frac d2
    \right)
    \ge
    \frac{1}{16\kappa^4}.
\]
Consequently,
\[
    \mathbb E\|S_n\|_2
    \ge
    \sqrt{\frac d2}\,
    \mathbb P
    \left(
        \|S_n\|_2^2\ge\frac d2
    \right)
    \ge
    c_\kappa\sqrt d.
\]
Since \(w_G(H_r)=r\mathbb E\|g\|_2\le r\sqrt d\),
\[
    A
    \ge
    c_\kappa
    \frac{w_G(H_r)}{\sqrt n}.
\]
Lemma~\ref{lem:rem-complexity} and \(C\sigma_Hr\le c_\kappa/2\) give
\[
    B
    \le
    \frac{c_\kappa}{2}
    \frac{w_G(H_r)}{\sqrt n}.
\]
The conclusion follows from the lower bound \(A-B\).
\end{proof}

\begin{remark}[Scope of the lower bound]
\label{rem:scope}
Condition~\ref{fr:score} holds for a correctly specified negative log-likelihood at the true parameter, where the gradient covariance is the Fisher information matrix. For a general \(M\)-estimation loss, the gradient covariance and expected Hessian need not coincide; the same argument can instead be formulated using the gradient covariance.

Theorem~\ref{thm:fisher-regular-lower} concerns Fisher balls, for which \(G^{1/2}H_r=rB_2^d\), and therefore establishes the scale
\[
    \frac{w_G(H_r)}{\sqrt n}
    \asymp
    r\sqrt{\frac dn}.
\]
It neither asserts a lower bound for arbitrary structured sets nor gives a minimax characterization in terms of Fisher width. Extending the argument to local cones or other structured sets would require additional assumptions on the score process beyond isotropy and a fourth-moment bound.
\end{remark}

\begin{corollary}[Examples satisfying Fisher regularity]
\label{cor:verification}

Suppose the corresponding model-specific assumptions of Appendix~\ref{app:verification} hold. In particular, assume bounded Fisher leverage for correctly specified logistic regression, the stated
local leverage and fourth-moment conditions for canonical-link generalized linear models, and
\[
    \mathbb E\bigl(X^\top G^{-1}X\bigr)^2<\infty
\]
for Gaussian linear regression. Then Conditions~\textup{(FR1)--(FR3)} hold for the respective models.
\end{corollary}

\section{Inverse-Fisher Recovery}
\label{sec:fisher_escape}

\subsection{Gaussian sensing and conic reduction}
\label{subsec:inverse_fisher_conic}

The learning results of the previous section involve fluctuations with covariance \(G\). We now consider Gaussian measurements with covariance \(G^{-1}\). A simple source of such measurements is the location model \(Z\sim N(\theta,G^{-1})\), whose Fisher information for \(\theta\) is \(G\). If \(Z_i,Z_i'\) are independent observations, then
\[
    a_i:=\frac{Z_i-Z_i'}{\sqrt2}\sim N(0,G^{-1}).
\]
Thus paired differences generate sensing rows of the form \(a_i=G^{-1/2}g_i\), where \(g_i\sim N(0,I_d)\), and hence a sensing operator \(AG^{-1/2}\) with \(A\) standard Gaussian. This construction motivates the covariance \(G^{-1}\); the results below apply to any Gaussian design with this covariance.

For a nonempty compact set \(T\subset\mathbb R^d\), define
\[
    r_{G^{-1}}(T)
    :=
    \inf_{v\in T}\|G^{-1/2}v\|_2,
\]
and, when \(r_{G^{-1}}(T)>0\),
\[
    \operatorname{rad}(G^{-1/2}T)
    :=
    \left\{
        \frac{G^{-1/2}v}{\|G^{-1/2}v\|_2}
        :v\in T
    \right\}.
\]
For a convex regularizer \(R:\mathbb R^d\to\mathbb R\), its descent cone at \(v_0\) is
\[
    D(R,v_0)
    :=
    \operatorname{cl}
    \left\{
        h:
        R(v_0+\alpha h)\le R(v_0)
        \text{ for some }\alpha>0
    \right\}.
\]

\begin{proposition}[Inverse-Fisher conic escape]
\label{prop:inverse_fisher_conic_escape}
Let \(G\succ0\), let \(A\in\mathbb R^{m\times d}\) have i.i.d.\
\(N(0,1)\) entries, and let \(C\subset\mathbb R^d\) be a nonzero
closed cone. Set
\( a_m:=\mathbb E\|g\|_2,
    \quad
    g\sim N(0,I_m).\)
Then, for every \(t>0\), with probability at least
\(1-e^{-t^2/2}\), we have
\[
    \inf_{\substack{v\in C\\
                    \|G^{-1/2}v\|_2=1}}
    \|AG^{-1/2}v\|_2
    \ge
    a_m
    -
    w\bigl(G^{-1/2}C\cap\mathbb S^{d-1}\bigr)
    -
    t.
\]
In particular, we have 
\(
    \ker(AG^{-1/2})\cap C=\{0\}
\)
whenever
\(
    a_m > w\bigl(G^{-1/2}C\cap\mathbb S^{d-1}\bigr)+t.
\)
\end{proposition}

\begin{proof}
Set
\[  
   T := C\cap \{v:\|G^{-1/2}v\|_2=1\}.
\]
Then
\[
    G^{-1/2}T
    =
    G^{-1/2}C\cap\mathbb S^{d-1}.
\]
Applying Gordon's escape theorem \citep{gordon1988}; see also \citet{vershynin2018}, to the set
\(G^{-1/2}C\cap\mathbb S^{d-1}\) gives
\[
    \inf_{\substack{v\in C\\
     \|G^{-1/2}v\|_2=1}}
    \|AG^{-1/2}v\|_2 \ge a_m - w\bigl(G^{-1/2}C\cap\mathbb S^{d-1}\bigr) - t.
\]
If the right-hand side is positive, then no nonzero vector in \(C\) belongs to \(\ker(AG^{-1/2})\), and hence
\[
    \ker(AG^{-1/2})\cap C=\{0\}.
\]
\end{proof}

Since \(a_m\asymp\sqrt m\), Proposition~\ref{prop:inverse_fisher_conic_escape} gives the sufficient measurement scale
\[
    m
    \gtrsim
    w\bigl(G^{-1/2}C\cap\mathbb S^{d-1}\bigr)^2.
\]
This is a sufficient escape bound. The approximate conic kinematic formula of \citet{amelunxen2014} instead locates the sharp Gaussian transition near
\[
    m=\delta(G^{-1/2}C).
\]

\begin{corollary}[Unique convex recovery]
\label{cor:inverse_fisher_recovery}
Let \(R:\mathbb R^d\to\mathbb R\) be convex and let \(C=D(R,v_0)\). Under the condition
\[
    a_m > w\bigl(G^{-1/2}C\cap\mathbb S^{d-1}\bigr)+t,
\]
the point \(v_0\) is, with probability at least \(1-e^{-t^2/2}\), the unique solution of
\[
    \min_{v\in\mathbb R^d}R(v)
    \quad\text{subject to}\quad
    AG^{-1/2}v=AG^{-1/2}v_0.
\]
\end{corollary}

\subsection{Fisher-induced weighted \texorpdfstring{\(\ell_1\)}{l1}
geometry}
\label{subsec:anisotropic_sparse_recovery}

We now specialize to unweighted \(\ell_1\) recovery with a diagonal Fisher matrix
\[
    G=\operatorname{diag}(\gamma_1,\ldots,\gamma_d),
    \qquad
    \gamma_i>0.
\]
After whitening the measurement operator, the anisotropy appears as a deterministic weight profile in the transformed descent cone. For \(S\subset[d]\), define
\begin{equation}
\label{eq:UG-definition}
    U_G(S)
    :=
    \inf_{\tau\ge0}
    \left[
        \sum_{i\in S}(1+\tau^2\gamma_i)
        +
        \sum_{j\notin S}
        \mathbb E
        \bigl(
            |g_j|-\tau\sqrt{\gamma_j}
        \bigr)_+^2
    \right],
\end{equation}
where \(g_1,\ldots,g_d\) are independent \(N(0,1)\) variables.

\begin{remark}[Scale invariance]
\label{rem:UG-scale-invariance}
For every \(c>0\),
\[
    U_{cG}(S)=U_G(S).
\]
Indeed, the change of variables \(\widetilde\tau=\sqrt c\,\tau\) leaves \eqref{eq:UG-definition} unchanged. Thus the noiseless recovery functional depends only on the relative Fisher weights.
\end{remark}

\begin{lemma}[Fisher-induced weighted-\(\ell_1\) representation]
\label{lem:anisotropic_sparse}
\label{thm:anisotropic_sparse}
Let \(x^\star\in\mathbb R^d\) be nonzero, with support \(S=\operatorname{supp}(x^\star)\), and set
\[
    C:=D(\|\cdot\|_1,x^\star),
    \qquad
    f_G(x):=\|G^{1/2}x\|_1,
    \qquad
    x_0:=G^{-1/2}x^\star.
\]
Then
\[
    G^{-1/2}C=D(f_G,x_0)
\]
and
\begin{equation}
    \delta(G^{-1/2}C)
    \le
    U_G(S).
    \label{eq:anisotropic_bound}
\end{equation}
\end{lemma}

\begin{proof}
The cone identity follows directly from the change of variables
\(x=G^{-1/2}v\):
\[
\begin{aligned}
    v\in D(\|\cdot\|_1,x^\star)
    &\iff
    \|x^\star+\alpha v\|_1\le\|x^\star\|_1
    \quad\text{for some }\alpha>0 \\
    &\iff
    f_G(x_0+\alpha G^{-1/2}v)\le f_G(x_0).
\end{aligned}
\]
The standard descent-cone bound gives
\[
    \delta(D(f_G,x_0))
    \le
    \inf_{\tau\ge0}
    \mathbb E
    \operatorname{dist}^2
    \bigl(g,\tau\partial f_G(x_0)\bigr).
\]
Here
\[
    \partial f_G(x_0)
    =
    \left\{
        z:
        z_i=\sqrt{\gamma_i}\operatorname{sign}(x_i^\star)
        \text{ for }i\in S,\;
        |z_j|\le\sqrt{\gamma_j}
        \text{ for }j\notin S
    \right\}.
\]
The coordinatewise distance computation is the standard weighted-\(\ell_1\) descent-cone recipe. For \(i\in S\),
\[
    \mathbb E
    \left(
        g_i-\tau\sqrt{\gamma_i}
        \operatorname{sign}(x_i^\star)
    \right)^2
    =
    1+\tau^2\gamma_i,
\]
whereas, for \(j\notin S\),
\[
    \mathbb E
    \operatorname{dist}^2
    \left(
        g_j,
        [-\tau\sqrt{\gamma_j},\tau\sqrt{\gamma_j}]
    \right)
    =
    \mathbb E
    \bigl(
        |g_j|-\tau\sqrt{\gamma_j}
    \bigr)_+^2.
\]
Summing over the coordinates and optimizing over \(\tau\ge0\) proves \eqref{eq:anisotropic_bound}.
\end{proof}

\begin{remark}[Two sources of anisotropic recovery cost]
\label{rem:two_bottlenecks}
The active and inactive coordinates enter \(U_G(S)\) differently. For \(i\in S\), a large \(\gamma_i\) reduces the scale \(\gamma_i^{-1/2}\) of the corresponding active column of \(AG^{-1/2}\). For \(j\notin S\), a small \(\gamma_j\) lowers the threshold \(\tau\sqrt{\gamma_j}\) and increases the Gaussian tail contribution. When \(G=I_d\), the distinction disappears and \(U_G(S)\) reduces to the usual isotropic \(\ell_1\) expression.
\end{remark}

\subsection{A two-sided anisotropic estimate}
\label{subsec:two_sided_anisotropic}

The functional \(U_G(S)\) is the standard weighted-\(\ell_1\) distance-to-subdifferential upper estimate and is not new in functional form. The additional step below is to exploit the fact that the descent cone and subdifferential depend only on the support and signs of the nonzero coordinates. Optimizing the ALMT error term over their magnitudes yields an explicit additive error determined by the Fisher mass on the active support.

\begin{theorem}[Two-sided anisotropic estimate]
\label{thm:two-sided}
Under the assumptions of Lemma~\ref{lem:anisotropic_sparse},
\[
    U_G(S)
    -
    2
    \sqrt{
        \frac{\operatorname{Tr}(G)}
             {\sum_{i\in S}\gamma_i}
    }
    \le
    \delta(G^{-1/2}C)
    \le
    U_G(S).
\]
Consequently, if
\[
    U_G(S)
    \ge
    4
    \sqrt{
        \frac{\operatorname{Tr}(G)}
             {\sum_{i\in S}\gamma_i}
    },
\]
then
\[
    \frac12U_G(S)
    \le
    \delta(G^{-1/2}C)
    \le
    U_G(S).
\]
\end{theorem}

\begin{proof}
The upper bound is Lemma~\ref{lem:anisotropic_sparse}. For the lower bound, the descent-cone error estimate of \citet{amelunxen2014} gives
\begin{equation}
\label{eq:almt-error}
    \delta(D(f_G,x))
    \ge
    \inf_{\tau\ge0}
    \mathbb E
    \operatorname{dist}^2
    \bigl(g,\tau\partial f_G(x)\bigr)
    -
    \frac{
        2\sup_{z\in\partial f_G(x)}\|z\|_2
    }{
        f_G(x/\|x\|_2)
    }.
\end{equation}

For the weighted \(\ell_1\) norm, both \(D(f_G,x)\) and \(\partial f_G(x)\) depend only on the support and signs of the nonzero coordinates of \(x\), not on their magnitudes. We may therefore apply \eqref{eq:almt-error} to any unit vector with support \(S\) and the same sign pattern as \(x_0\). For every such vector, the first term on the right-hand side is \(U_G(S)\), while
\[
    \sup_{z\in\partial f_G(x)}\|z\|_2
    =
    \sqrt{\operatorname{Tr}(G)}.
\]

For any unit vector supported on \(S\), Cauchy--Schwarz gives
\[
    f_G(x)
    =
    \sum_{i\in S}\sqrt{\gamma_i}|x_i|
    \le
    \left(
        \sum_{i\in S}\gamma_i
    \right)^{1/2}.
\]
Equality holds for
\[
    x_i
    =
    \frac{
        \operatorname{sign}(x_i^\star)\sqrt{\gamma_i}
    }{
        \left(
            \sum_{j\in S}\gamma_j
        \right)^{1/2}
    },
    \qquad i\in S,
\]
with \(x_j=0\) for \(j\notin S\). Substitution into \eqref{eq:almt-error} yields
\[
    \delta(G^{-1/2}C)
    \ge
    U_G(S)
    -
    2
    \sqrt{
        \frac{\operatorname{Tr}(G)}
             {\sum_{i\in S}\gamma_i}
    }.
\]
The factor-two estimate follows from the stated condition.
\end{proof}

\begin{remark}[Regime of the estimate]
\label{rem:two_sided_regime}
The additive error is small when the Fisher mass on the active support is not negligible relative to \(\operatorname{Tr}(G)\). In proportional regimes with controlled Fisher mass ratios, the factor-two condition is readily satisfied. In extreme sparsity, the additive term may exceed \(U_G(S)\), in which case the lower bound is vacuous; the theorem does not provide a uniform multiplicative comparison.
\end{remark}

\begin{conjecture}[Extreme-sparsity regime]
\label{conj:extreme_sparsity}
There exists a universal constant \(c>0\) such that, for every diagonal \(G\succ0\), every support \(S\subset[d]\), and any nonzero \(x^\star\) supported on \(S\),
\[
    \delta\!\left(
        G^{-1/2}D(\|\cdot\|_1,x^\star)
    \right)
    \ge
    c\,U_G(S).
\]
Together with Lemma~\ref{lem:anisotropic_sparse}, this would give
\[
    \delta\!\left(
        G^{-1/2}D(\|\cdot\|_1,x^\star)
    \right)
    \asymp
    U_G(S)
\]
without the regime condition of Theorem~\ref{thm:two-sided}.
\end{conjecture}

Section~\ref{subsec:exp-anisotropic} reports one sparse configuration for which the additive lower bound is vacuous and compares \(U_G(S)\) with the observed transition. The experiment is consistent with the conjectured comparison in this configuration, but does not address its uniform validity.

\subsection{Support-dependent upper recovery functional}
\label{subsec:support_selection_regularization}

The upper functional \(U_G(S)\) depends on the location of the support as well as its cardinality. We now derive a coordinatewise decomposition that induces a natural ordering of supports according to
their Fisher profile.

\begin{proposition}[Support-dependent Fisher recovery complexity]
\label{prop:support_dependent_regularization}
Let
\[
    G=\operatorname{diag}(\gamma_1,\ldots,\gamma_d),
    \qquad
    \gamma_i>0.
\]
For \(\tau\ge0\), define
\[
    q_i(\tau)
    :=
    1+\tau^2\gamma_i
    -
    \mathbb E
    \bigl(
        |g|-\tau\sqrt{\gamma_i}
    \bigr)_+^2,
\]
and
\[
    B(\tau)
    :=
    \sum_{j=1}^d
    \mathbb E
    \bigl(
        |g|-\tau\sqrt{\gamma_j}
    \bigr)_+^2,
    \qquad
    g\sim N(0,1).
\]
Then:

\begin{enumerate}[label=(\roman*)]

\item \textup{(Support-cost representation)}
For every \(S\subset[d]\),
\[
    U_G(S)
    =
    \inf_{\tau\ge0}
    \left[
        B(\tau)
        +
        \sum_{i\in S}q_i(\tau)
    \right].
\]

\item \textup{(Nonnegativity and spectral monotonicity)}
For every fixed \(\tau\ge0\), \(q_i(\tau)\ge0\), and
\(q_i(\tau)\) is nondecreasing as a function of \(\gamma_i\).

\item \textup{(Nested-support monotonicity)}
If \(S_2\subseteq S_1\), then
\[
    U_G(S_2)\le U_G(S_1).
\]

\item \textup{(Support comparison)}
If \(S_1,S_2\subset[d]\) satisfy
\[
    \sum_{i\in S_2}q_i(\tau)
    \ge
    \sum_{i\in S_1}q_i(\tau)
    \qquad
    \text{for every }\tau\ge0,
\]
then
\[
    U_G(S_2)\ge U_G(S_1).
\]

\end{enumerate}
\end{proposition}

\begin{proof}
For \(S\subset[d]\), let
\[
    F_S(\tau)
    :=
    \sum_{i\in S}(1+\tau^2\gamma_i)
    +
    \sum_{j\notin S}
    \mathbb E
    \bigl(
        |g|-\tau\sqrt{\gamma_j}
    \bigr)_+^2.
\]
Adding and subtracting the off-support contribution for each
\(i\in S\) gives
\[
    F_S(\tau)
    =
    B(\tau)
    +
    \sum_{i\in S}q_i(\tau),
\]
and taking the infimum over \(\tau\ge0\) proves~\textup{(i)}.

Since \((|g|-a)_+^2\le g^2\) for every \(a\ge0\),
\[
    q_i(\tau)
    \ge
    \tau^2\gamma_i
    \ge0.
\]
For fixed \(\tau\), the term \(1+\tau^2\gamma_i\) is nondecreasing in \(\gamma_i\), whereas
\(\mathbb E(|g|-\tau\sqrt{\gamma_i})_+^2\) is nonincreasing. Thus \(q_i(\tau)\) is nondecreasing in \(\gamma_i\), proving \textup{(ii)}. Parts~\textup{(iii)} and~\textup{(iv)} follow from the corresponding pointwise ordering of \(F_S(\tau)\) and taking infima over \(\tau\ge0\).
\end{proof}

\begin{corollary}[Curvature-increasing support swaps]
\label{cor:curvature_increasing_swap}
Let \(S_1,S_2\subset[d]\) have the same cardinality. Write
\[
    S_2\setminus S_1=\{i_1,\ldots,i_m\},
    \qquad
    S_1\setminus S_2=\{j_1,\ldots,j_m\}.
\]
If the indices can be ordered so that
\[
    \gamma_{i_\ell}\ge\gamma_{j_\ell}
    \qquad
    \text{for every }\ell=1,\ldots,m,
\]
then
\[
    U_G(S_2)\ge U_G(S_1).
\]
In particular, the conclusion holds if
\[
    \min_{i\in S_2\setminus S_1}\gamma_i
    \ge
    \max_{j\in S_1\setminus S_2}\gamma_j.
\]
\end{corollary}

\begin{proof}
By Proposition~\ref{prop:support_dependent_regularization}\textup{(ii)},
\[
    q_{i_\ell}(\tau)\ge q_{j_\ell}(\tau)
    \qquad
    \text{for every }\ell
    \text{ and }\tau\ge0.
\]
Summing over the exchanged coordinates and applying Proposition~\ref{prop:support_dependent_regularization}\textup{(iv)}
gives the result.
\end{proof}

\begin{remark}[Support geometry]
When \(G=I_d\), the functions \(q_i\) are independent of \(i\), so \(U_G(S)\) depends only on \(|S|\). For anisotropic \(G\), equal-cardinality supports can have different upper recovery
complexities. Whether Fisher-aware regularizers can systematically exploit this ordering remains open.
\end{remark}

\section{Primal--Inverse Width Inequalities}
\label{sec:uncertainty}

The Fisher and inverse-Fisher widths are opposite linear deformations of a common coordinate set. The main result of this section is the sharp inequality
\[
    w_G(T)w_{G^{-1}}(T)\ge w(T)^2,
\]
obtained from a log-convexity property for commuting positive-definite matrices. We then apply the inequality to a common localized cone and derive a noncommutative geometric-mean extension.

\subsection{Commuting log-convexity}
\label{subsec:commuting_uncertainty}

\begin{theorem}[Log-convexity for commuting Fisher metrics]
\label{thm:power_log_convexity}
Let \(T\subset\mathbb R^d\) be nonempty and compact, and let \(G_0,G_1\succ0\) commute. For \(\theta\in[0,1]\), define
\[
    G_\theta:=G_0^{1-\theta}G_1^\theta.
\]
Then
\begin{equation}
\label{eq:commuting-log-convexity}
    w_{G_\theta}(T)
    \le
    w_{G_0}(T)^{1-\theta}
    w_{G_1}(T)^\theta.
\end{equation}
Consequently,
\begin{equation}
\label{eq:commuting-midpoint}
    w_{G_0}(T)w_{G_1}(T)
    \ge
    w_{G_0\#G_1}(T)^2,
\end{equation}
where \(G_0\#G_1\) is the affine-invariant geometric mean.
\end{theorem}

\begin{proof}
Gaussian width is translation invariant:
\[
    w(T+a)=w(T),
\]
since \(\mathbb E\langle g,a\rangle=0\). Thus the comparison below depends only on Gaussian increments.

If \(T\) is a singleton, all widths vanish. Otherwise, \(w_H(T)>0\) for every \(H\succ0\), so the optimization below is well-defined. Fix \(0<\theta<1\), since the endpoint cases are immediate, and set
\[
    H_0:=G_0^{1/2},
    \qquad
    H_1:=G_1^{1/2},
    \qquad
    H_\theta:=G_\theta^{1/2}.
\]
Because \(G_0\) and \(G_1\) commute, they are simultaneously diagonalizable, and in their common eigenbasis
\[
    H_\theta=H_0^{1-\theta}H_1^\theta.
\]

For \(r>0\), define
\[
    L_r
    :=
    (1-\theta)rH_0
    +
    \theta r^{-(1-\theta)/\theta}H_1.
\]
The scalar weighted arithmetic--geometric mean inequality, applied coordinatewise in the common eigenbasis, gives
\[
    L_r\succeq H_\theta.
\]
Since \(L_r\) and \(H_\theta\) commute and are simultaneously diagonalizable with positive eigenvalues,
\[
    L_r^2\succeq H_\theta^2.
\]
Hence, for all \(u,v\in T\),
\[
    \|L_r(u-v)\|_2^2
    \ge
    \|H_\theta(u-v)\|_2^2.
\]

Let
\[
    X_v:=\langle g,H_\theta v\rangle,
    \qquad
    Y_v:=\langle g,L_rv\rangle,
    \qquad
    v\in T.
\]
The increment comparison and the Sudakov--Fernique theorem imply
\[
    w_{G_\theta}(T)
    =
    \mathbb E\sup_{v\in T}X_v
    \le
    \mathbb E\sup_{v\in T}Y_v.
\]
By subadditivity of the supremum,
\[
\begin{aligned}
    \mathbb E\sup_{v\in T}Y_v
    &\le
    (1-\theta)r\,w_{G_0}(T)
    +
    \theta r^{-(1-\theta)/\theta}w_{G_1}(T).
\end{aligned}
\]
Optimizing over \(r>0\), with
\[
    r
    =
    \left(
        \frac{w_{G_1}(T)}
             {w_{G_0}(T)}
    \right)^\theta,
\]
gives \eqref{eq:commuting-log-convexity}.

Taking \(\theta=1/2\) yields
\[
    w_{G_{1/2}}(T)^2
    \le
    w_{G_0}(T)w_{G_1}(T).
\]
For commuting matrices,
\[
    G_{1/2}
    =
    G_0^{1/2}G_1^{1/2}
    =
    G_0\#G_1,
\]
which proves \eqref{eq:commuting-midpoint}.
\end{proof}

\begin{corollary}[Log-convexity along the power geodesic]
\label{cor:power-geodesic-log-convexity}
Let \(T\subset\mathbb R^d\) be nonempty and compact, and let \(G\succ0\). Then
\[
    F(\alpha):=w_{G^\alpha}(T),
    \qquad
    \alpha\in\mathbb R,
\]
is log-convex. In particular, for every \(\alpha_0,\alpha_1\in\mathbb R\) and \(\theta\in[0,1]\),
\[
    F\bigl((1-\theta)\alpha_0+\theta\alpha_1\bigr)
    \le
    F(\alpha_0)^{1-\theta}F(\alpha_1)^\theta,
\]
and, for every \(\alpha\in\mathbb R\),
\begin{equation}
\label{eq:power-primal-dual}
    w_{G^\alpha}(T)w_{G^{-\alpha}}(T)
    \ge
    w(T)^2.
\end{equation}
\end{corollary}

\begin{proof}
Apply Theorem~\ref{thm:power_log_convexity} with \(G_0=G^{\alpha_0}\) and \(G_1=G^{\alpha_1}\), which commute. Taking \(\alpha_0=\alpha\), \(\alpha_1=-\alpha\), and \(\theta=1/2\) gives \eqref{eq:power-primal-dual}.
\end{proof}

\subsection{The sharp primal--inverse product inequality}
\label{subsec:sharp_product}

\begin{theorem}[Sharp primal--inverse width product inequality]
\label{thm:primal_dual_uncertainty}
Let \(T\subset\mathbb R^d\) be nonempty and compact, and let \(G\succ0\). Then
\begin{equation}
\label{eq:sharp-primal-dual}
    w_G(T)w_{G^{-1}}(T)
    \ge
    w(T)^2.
\end{equation}
Equivalently,
\begin{equation}
\label{eq:sharp-linear-tradeoff}
    \inf_{s>0}
    \left\{
        s\,w_G(T)+s^{-1}w_{G^{-1}}(T)
    \right\}
    \ge
    2w(T).
\end{equation}
\end{theorem}

\begin{proof}
The product inequality is Corollary~\ref{cor:power-geodesic-log-convexity} with \(\alpha=1\).
For \(s>0\), the arithmetic--geometric mean inequality gives
\[
    s\,w_G(T)+s^{-1}w_{G^{-1}}(T)
    \ge
    2\sqrt{w_G(T)w_{G^{-1}}(T)}
    \ge
    2w(T).
\]
Conversely,
\[
    \inf_{s>0}
    \left\{
        s\,w_G(T)+s^{-1}w_{G^{-1}}(T)
    \right\}
    =
    2\sqrt{w_G(T)w_{G^{-1}}(T)},
\]
so the product and linear forms are equivalent.
\end{proof}

\begin{remark}[Sharpness]
\label{rem:sharp}
The constant \(1\) in \eqref{eq:sharp-primal-dual} is optimal. If
\(G=\lambda I_d\), then
\[
    w_G(T)=\sqrt{\lambda}\,w(T),
    \qquad
    w_{G^{-1}}(T)=\lambda^{-1/2}w(T),
\]
and equality holds for every compact \(T\).

Equality may also occur for anisotropic \(G\). If \(T=\{v,-v\}\), then
\[
    \frac{w_G(T)w_{G^{-1}}(T)}{w(T)^2}
    =
    \frac{
        \|G^{1/2}v\|_2\|G^{-1/2}v\|_2
    }{
        \|v\|_2^2
    }
    \ge1.
\]
The inequality is Cauchy--Schwarz, and equality holds precisely when \(v\) belongs to an eigenspace of \(G\).
\end{remark}

\subsection{Localized primal--inverse trade-offs}
\label{subsec:common_cone_tradeoff}

Let \(C\subset\mathbb R^d\) be a nonzero closed convex cone. Since an unbounded cone has infinite Gaussian width, we use the common localization
\[
    T_C:=C\cap B_2^d.
\]
Set
\[
    D:=G^{-1/2}C,
\]
and define the restricted radial distortions
\[
    r_C
    :=
    \min_{v\in C\cap\mathbb S^{d-1}}
    \|G^{-1/2}v\|_2,
    \qquad
    R_C
    :=
    \max_{v\in C\cap\mathbb S^{d-1}}
    \|G^{-1/2}v\|_2.
\]
Then \(0<r_C\le R_C<\infty\).

\begin{proposition}[Restricted distortion on a cone]
\label{prop:restricted_cone_comparison}
With the notation above,
\begin{equation}
\label{eq:restricted_width_comparison}
    r_C\,w(D\cap B_2^d)
    \le
    w_{G^{-1}}(T_C)
    \le
    R_C\,w(D\cap B_2^d).
\end{equation}
Consequently,
\begin{equation}
\label{eq:restricted_delta_comparison}
    r_C\sqrt{(\delta(D)-1)_+}
    \le
    w_{G^{-1}}(T_C)
    \le
    R_C\sqrt{\delta(D)}.
\end{equation}
\end{proposition}

\begin{proof}
Set
\[
    K_G:=G^{-1/2}T_C.
\]
Along each ray of \(D\), the radial extent of \(K_G\) lies in \([r_C,R_C]\). Hence
\[
    r_C(D\cap B_2^d)
    \subseteq
    K_G
    \subseteq
    R_C(D\cap B_2^d).
\]
Monotonicity and homogeneity of Gaussian width give \eqref{eq:restricted_width_comparison}.

For a closed convex cone \(D\),
\[
    w(D\cap B_2^d)
    =
    \mathbb E\|\Pi_Dg\|_2,
    \qquad
    \delta(D)
    =
    \mathbb E\|\Pi_Dg\|_2^2.
\]
Jensen's inequality and the Gaussian Poincar\'e inequality yield
\[
    \delta(D)-1
    \le
    w(D\cap B_2^d)^2
    \le
    \delta(D).
\]
Combining this with \eqref{eq:restricted_width_comparison} proves \eqref{eq:restricted_delta_comparison}.
\end{proof}

\begin{corollary}[Localized primal--inverse trade-off]
\label{cor:common_cone_tradeoff}
Let \(C\subset\mathbb R^d\) be a nonzero closed convex cone. Then
\begin{equation}
\label{eq:common_cone_width_tradeoff}
    w_G(T_C)
    \sqrt{\delta(G^{-1/2}C)}
    \ge
    \frac{w(T_C)^2}{R_C}.
\end{equation}
In particular,
\begin{equation}
\label{eq:common_cone_delta_tradeoff}
    w_G(T_C)
    \sqrt{\delta(G^{-1/2}C)}
    \ge
    \frac{(\delta(C)-1)_+}{R_C}.
\end{equation}
\end{corollary}

\begin{proof}
Applying Theorem~\ref{thm:primal_dual_uncertainty} to \(T_C\) gives
\[
    w_G(T_C)w_{G^{-1}}(T_C)\ge w(T_C)^2.
\]
By Proposition~\ref{prop:restricted_cone_comparison},
\[
    w_{G^{-1}}(T_C)
    \le
    R_C\sqrt{\delta(G^{-1/2}C)}.
\]
This proves \eqref{eq:common_cone_width_tradeoff}. Applying
\[
    w(C\cap B_2^d)^2\ge(\delta(C)-1)_+
\]
gives \eqref{eq:common_cone_delta_tradeoff}.
\end{proof}

\begin{remark}[Scope]
The corollary compares the Fisher width and inverse-Fisher recovery geometry of the same localized coordinate object \(C\cap B_2^d\). It is not a universal duality between arbitrary learning and recovery problems. Same-set comparisons are understood as in Remark~\ref{rem:fixed_chart}.
\end{remark}

\subsection{Noncommuting metrics}
\label{subsec:noncommutative_uncertainty}

For arbitrary \(G_1,G_2\succ0\), define their affine-invariant geometric mean by
\[
    G_1\#G_2
    :=
    G_1^{1/2}
    \left(
        G_1^{-1/2}G_2G_1^{-1/2}
    \right)^{1/2}
    G_1^{1/2}.
\]

\begin{lemma}[Matrix arithmetic--geometric mean]
\label{lem:matrix_amgm}
Let \(G_1,G_2\succ0\). Then:

\begin{enumerate}[label=(\roman*)]

\item
For \(a,b>0\),
\[
    (aG_1)\#(bG_2)
    =
    \sqrt{ab}\,(G_1\#G_2).
\]

\item
For every \(s>0\),
\[
    s^2G_1+s^{-2}G_2
    \succeq
    2(G_1\#G_2).
\]

\end{enumerate}
\end{lemma}

\begin{proof}
Part~\textup{(i)} follows directly from the definition. For part~\textup{(ii)}, set
\(C:=G_1^{-1/2}G_2G_1^{-1/2}\). Then
\[
    G_1+G_2-2(G_1\#G_2)
    =
    G_1^{1/2}(I-C^{1/2})^2G_1^{1/2}
    \succeq0.
\]
Apply this inequality to \((s^2G_1,s^{-2}G_2)\) and use part~\textup{(i)}.
\end{proof}

\begin{theorem}[Noncommutative geometric-mean bound]
\label{thm:geometric_mean_uncertainty}
Let \(T\subset\mathbb R^d\) be nonempty and compact, and let \(G_1,G_2\succ0\). Then
\begin{equation}
\label{eq:geometric-mean-product}
    w_{G_1}(T)w_{G_2}(T)
    \ge
    \frac12
    w_{G_1\#G_2}(T)^2.
\end{equation}
Equivalently,
\begin{equation}
\label{eq:geometric-mean-linear}
    \inf_{s>0}
    \left\{
        s\,w_{G_1}(T)+s^{-1}w_{G_2}(T)
    \right\}
    \ge
    \sqrt2\,w_{G_1\#G_2}(T).
\end{equation}
\end{theorem}

\begin{proof}
If \(T\) is a singleton, the claim is immediate. Otherwise, let \(g_1,g_2,g\sim N(0,I_d)\) be independent and, for \(s>0\), define
\[
    Z_v^{(s)}
    :=
    s\langle g_1,G_1^{1/2}v\rangle
    +
    s^{-1}\langle g_2,G_2^{1/2}v\rangle,
\]
and
\[
    W_v
    :=
    \sqrt2\,
    \langle g,(G_1\#G_2)^{1/2}v\rangle.
\]
For \(h=u-v\), Lemma~\ref{lem:matrix_amgm} gives
\[
\begin{aligned}
    \mathbb E
    |Z_u^{(s)}-Z_v^{(s)}|^2
    &=
    h^\top(s^2G_1+s^{-2}G_2)h \\
    &\ge
    2h^\top(G_1\#G_2)h \\
    &=
    \mathbb E|W_u-W_v|^2.
\end{aligned}
\]
Sudakov--Fernique therefore yields
\[
    \mathbb E\sup_{v\in T}Z_v^{(s)}
    \ge
    \sqrt2\,w_{G_1\#G_2}(T).
\]
On the other hand,
\[
    \mathbb E\sup_{v\in T}Z_v^{(s)}
    \le
    s\,w_{G_1}(T)+s^{-1}w_{G_2}(T).
\]
This proves \eqref{eq:geometric-mean-linear}; optimizing over \(s>0\) gives \eqref{eq:geometric-mean-product}.
\end{proof}

\section{Numerical Experiments}
\label{sec:experiments}

We report two controlled recovery experiments and a separate
illustration of primal--inverse width redistribution. The first
experiment compares the empirical transition of ordinary basis pursuit
with the support-dependent functional \(U_G(S)\). The second examines
the effect of deterministic weighting and finite-sample column
normalization. The final experiment visualizes the redistribution of
Gaussian width under the deformations \(G^{1/2}\) and \(G^{-1/2}\).
These experiments are intended as controlled illustrations of the
theory rather than as a broad empirical study of sparse-recovery phase
transitions.

\subsection{Support-dependent anisotropic recovery}
\label{subsec:exp-anisotropic}

\paragraph{Setup.}
We fixed the ambient dimension and sparsity at
\[
    d=256,
    \qquad
    k=16.
\]
For each diagonal Fisher matrix
\[
    G=\operatorname{diag}(\gamma_1,\ldots,\gamma_d),
\]
we fixed a \(k\)-sparse vector \(x^\star\) with prescribed support \(S=\operatorname{supp}(x^\star)\), drew \(A\in\mathbb R^{m\times d}\) with independent \(N(0,1)\) entries, and
formed the noiseless observations \(y=AG^{-1/2}x^\star.\) The nonzero entries satisfy
\(
    x_i^\star\in\{-1,+1\},
    \qquad i\in S,
\)
and are drawn independently and uniformly once at the beginning of the experiment. The resulting signal \(x^\star\) is fixed across all trials; only the measurement matrix \(A\) is resampled.

Since the noiseless recovery problem is invariant under the common rescaling \(G\mapsto cG\), each Fisher profile was normalized so that
\[
    \operatorname{Tr}(G)=d.
\]

We recovered \(x^\star\) by ordinary basis pursuit,
\[
    \widehat x
    \in
    \arg\min_{x\in\mathbb R^d}\|x\|_1
    \quad\text{subject to}\quad
    AG^{-1/2}x=y.
\]
The optimization problems were solved in CVXPY using the CLARABEL solver. Recovery was declared successful when
\[
    \frac{\|\widehat x-x^\star\|_2}
         {\|x^\star\|_2}
    <10^{-4}.
\]

For each Fisher profile and each value of \(m\), we ran \(200\) independent trials. Empirical recovery probabilities are reported with Wilson \(95\%\) confidence intervals. We define
\(\widehat m_{50}\) by linear interpolation between the two adjacent grid points whose empirical recovery probabilities bracket \(1/2\).

We considered five profiles:
\[
    \text{isotropic},\qquad
    \text{low-support},\qquad
    \text{high-support},\qquad
    \text{flat off-support},\qquad
    \text{mixed / one flat}.
\]
With \(S=\{1,\ldots,k\}\), the five diagonal profiles are defined, before the common trace normalization, as follows:
\begin{center}
\begin{tabular}{lll}
\toprule
Profile
&
\(\gamma_i,\ i\in S\)
&
\(\gamma_i,\ i\notin S\)
\\
\midrule
isotropic
& \(1\)
& \(1\)
\\
low-support
& \(0.25\)
& \(4\)
\\
high-support
& \(4\)
& \(0.25\)
\\
flat off-support
& \(1\)
& \(2\)
\\
mixed / one flat
& \(6\) for \(i=1\), \(1\) otherwise
& \(1.5\)
\\
\bottomrule
\end{tabular}
\end{center}
Each profile is subsequently rescaled by a common positive constant so that \(\operatorname{Tr}(G)=d\). The profiles separate curvature on the active support from the contribution of inactive coordinates.

\paragraph{Comparison with \(U_G(S)\).}
For each profile, we evaluated the upper functional \(U_G(S)\) and compared it with the interpolated empirical transition \(\widehat m_{50}\). The results are summarized in Table~\ref{tab:exp1-transitions}.

\begin{table}[tbp]
\centering
\caption{Theoretical functional and empirical transition for ordinary basis pursuit under the five Fisher profiles.}
\label{tab:exp1-transitions}
\begin{tabular}{lccc}
\hline
Profile
&
\(U_G(S)\)
&
\(\widehat m_{50}\)
&
\(\widehat m_{50}/U_G(S)\)
\\
\hline
isotropic
& \(61.1\) & \(60.1\) & \(0.984\) \\
low-support
& \(22.4\) & \(22.0\) & \(0.984\) \\
high-support
& \(178.0\) & \(177.7\) & \(0.999\) \\
flat off-support
& \(44.9\) & \(44.8\) & \(0.998\) \\
mixed / one flat
& \(57.5\) & \(57.0\) & \(0.991\) \\
\hline
\end{tabular}
\end{table}

Across all five profiles, \(U_G(S)\) captures both the ordering and the numerical location of the observed transition. The ratios satisfy
\[
    0.984
    \le
    \frac{\widehat m_{50}}{U_G(S)}
    \le
    0.999,
\]
so the discrepancy is below approximately \(2\%\) in every tested configuration.

The support dependence is substantial. Moving from the low-support to the high-support profile increases the empirical transition from about \(22\) to \(178\) measurements, although \(d\), \(k\), and the decoder remain unchanged. The high-support configuration therefore requires more than eight times as many measurements as the low-support configuration and nearly three times as many as the isotropic profile. Thus sparsity alone does not determine the observed recovery scale; the location of the support in the Fisher spectrum is also decisive.

The inactive coordinates also matter. The flat off-support profile has an empirical transition near \(44.8\), compared with \(60.1\) in the isotropic case. The mixed profile remains close to the isotropic transition, at approximately \(57.0\). These comparisons illustrate that the transition depends on the full weighted descent-cone geometry, rather than on the cardinality of the support or a single extreme coordinate.

\begin{figure}[tbp]
    \centering
    \includegraphics[width=\textwidth]
    {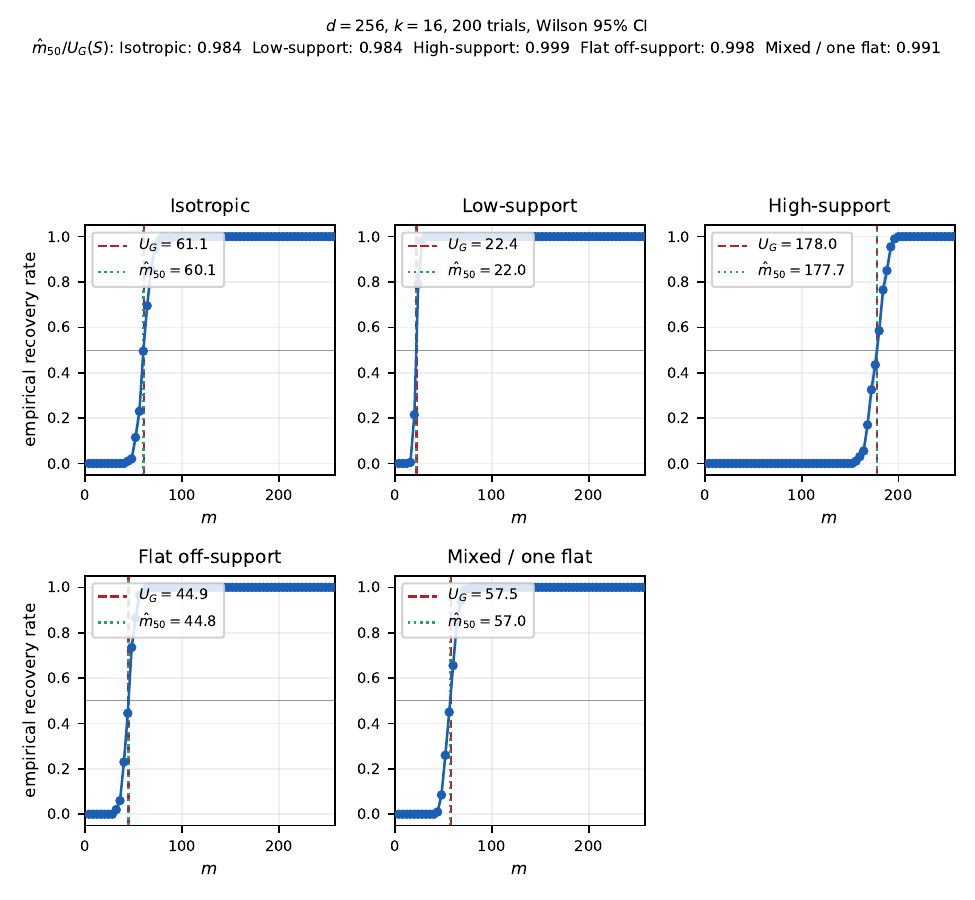}
    \caption{Empirical recovery curves for ordinary basis pursuit under
    inverse-Fisher measurements. We use \(d=256\), \(k=16\), and
    \(200\) independent trials for each value of \(m\) and each Fisher
    profile. Shaded bands are Wilson \(95\%\) confidence intervals.
    The dashed and dotted vertical lines mark \(U_G(S)\) and the
    interpolated empirical transition \(\widehat m_{50}\),
    respectively. Across all five profiles, \(U_G(S)\) captures both
    the ordering and the numerical location of the observed
    transitions.}
    \label{fig:recovery_curves}
\end{figure}

\subsection{Effect of decoder weighting and normalization}
\label{subsec:decoder-comparison}

The preceding experiment concerns ordinary basis pursuit in the original coordinates. We next examine how the transition changes when the decoder compensates for, or reinforces, the diagonal anisotropy.

Let
\[
    M:=AG^{-1/2}.
\]
We compared four decoders.

\paragraph{Unweighted basis pursuit.}
The baseline decoder is
\[
    \widehat x_{\mathrm{unw}}
    \in
    \arg\min_x\|x\|_1
    \quad\text{subject to}\quad
    Mx=y.
\]

\paragraph{Inverse-square-root weighting.}
The second decoder solves
\[
    \widehat x_{\mathrm{inv}}
    \in
    \arg\min_x
    \sum_{i=1}^d\gamma_i^{-1/2}|x_i|
    \quad\text{subject to}\quad
    Mx=y.
\]
Under the change of variables \(z=G^{-1/2}x\), this becomes ordinary basis pursuit for the isotropic system \(Az=y\). It therefore compensates for the population-level diagonal column scaling induced by \(G^{-1/2}\).

\paragraph{Square-root Fisher weighting.}
The third decoder solves
\[
    \widehat x_{\mathrm{F}}
    \in
    \arg\min_x
    \sum_{i=1}^d\gamma_i^{1/2}|x_i|
    \quad\text{subject to}\quad
    Mx=y.
\]
This Fisher-weighted heuristic penalizes high-curvature coordinates more strongly. It is included as a geometric comparison and is not claimed to be optimal for sparse recovery.

\paragraph{Column-normalized basis pursuit.}
For the fourth decoder, define
\[
    D_M
    :=
    \operatorname{diag}
    \bigl(
        \|M_1\|_2,\ldots,\|M_d\|_2
    \bigr),
    \qquad
    \widetilde M:=MD_M^{-1},
\]
where \(M_j\) denotes the \(j\)-th column of \(M\). We solve
\[
    \widehat z
    \in
    \arg\min_z\|z\|_1
    \quad\text{subject to}\quad
    \widetilde Mz=y,
\]
and transform back via
\[
    \widehat x_{\mathrm{col}}
    :=
    D_M^{-1}\widehat z.
\]
This decoder uses the realized finite-sample column norms rather than the population scales \(\gamma_i^{-1/2}\).

We used the same dimensions, profiles, recovery criterion, and solver as in the preceding experiment. For every decoder, Fisher profile, and value of \(m\), we ran \(200\) independent trials. The dashed vertical line in each panel of Figure~\ref{fig:decoder-comparison} marks
\(U_G(S)\), which is the theoretical functional for the unweighted decoder only.

The interpolated empirical transitions are reported in Table~\ref{tab:exp2-transitions}.

\begin{table}[tbp]
\centering
\caption{Interpolated empirical transitions for the four decoders.}
\label{tab:exp2-transitions}
\begin{tabular}{lcccc}
\hline
Profile
&
unweighted
&
\(\gamma_i^{-1/2}\)
&
\(\gamma_i^{1/2}\)
&
column-normalized
\\
\hline
isotropic
& \(60.3\) & \(60.8\) & \(61.0\) & \(59.5\) \\
low-support
& \(22.1\) & \(61.3\) & \(16.5\) & \(60.7\) \\
high-support
& \(177.6\) & \(60.9\) & \(247.1\) & \(60.7\) \\
flat off-support
& \(45.2\) & \(60.4\) & \(33.5\) & \(59.4\) \\
mixed / one flat
& \(57.1\) & \(60.9\) & \(71.0\) & \(59.7\) \\
\hline
\end{tabular}
\end{table}

In the isotropic profile, all four transitions lie near \(m=60\), as expected. Under anisotropy, unweighted basis pursuit ranges from approximately \(22\) measurements in the low-support profile to approximately \(178\) in the high-support profile.

Inverse-square-root weighting removes almost all profile dependence: its empirical transitions lie between \(60.4\) and \(61.3\). Finite-sample column normalization has nearly the same effect, with transitions between \(59.4\) and \(60.7\). The close agreement between these two decoders indicates that the dominant profile dependence in this experiment is associated with the diagonal column scaling.

This compensation is not uniformly beneficial. The low-support and flat off-support profiles are favorable for the unweighted decoder. Compensating for the anisotropy moves their transitions back toward the isotropic level and therefore increases the required number of measurements. Conversely, in the high-support profile, inverse-square-root weighting reduces the transition from approximately \(177.6\) to \(60.9\), while column normalization reduces it to approximately \(60.7\).

Square-root Fisher weighting reinforces the profile dependence. Its transition decreases to \(16.5\) in the low-support profile and to \(33.5\) in the flat off-support profile, but increases to approximately \(247.1\) in the high-support profile and \(71.0\) in the mixed profile. Thus a geometrically natural Fisher weighting need not be uniformly favorable for sparse recovery.

\begin{figure}[tbp]
    \centering
    \includegraphics[width=\textwidth]
    {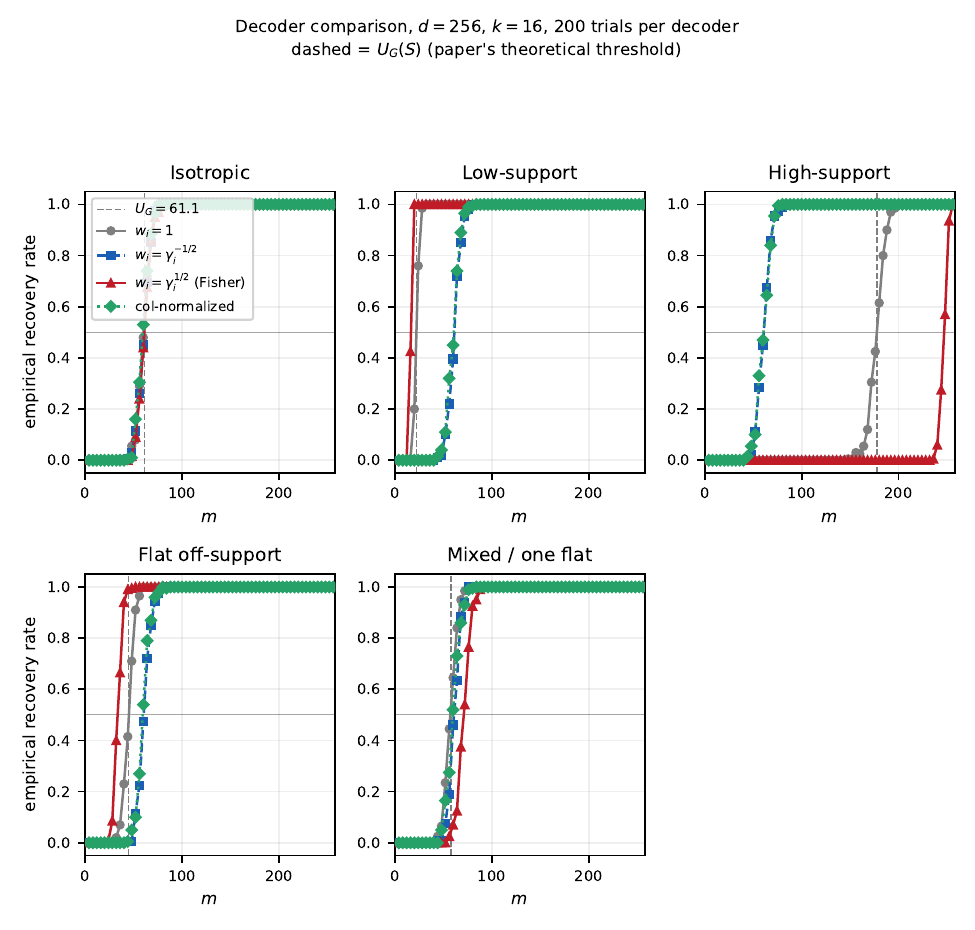}
    \caption{Comparison of four decoders under the five Fisher
    profiles, with \(d=256\), \(k=16\), and \(200\) independent trials
    for each decoder and each value of \(m\). The dashed vertical line
    marks \(U_G(S)\), which applies to the unweighted decoder.
    Inverse-square-root weighting and finite-sample column
    normalization largely remove the profile dependence, whereas
    square-root Fisher weighting reinforces it.}
    \label{fig:decoder-comparison}
\end{figure}

\subsection{Primal--inverse width redistribution}
\label{subsec:width-redistribution}

The final experiment illustrates how a metric deformation can redistribute Gaussian width between the primal and inverse geometries. It is not a numerical verification of Theorem~\ref{thm:primal_dual_uncertainty}, which is an exact inequality.

Let \(n=64\) and \(r=20\). We generated \(V\in\mathbb R^{n\times r}\) by QR-factorizing a standard Gaussian matrix and retained its orthonormal columns. Independently, we drew
\[
    s_i\sim\operatorname{Uniform}(0.5,5),
    \qquad i=1,\ldots,r,
\]
and set
\[
    G_0
    :=
    V\operatorname{diag}(s)V^\top+0.1I_n.
\]
We then defined
\[
    G_\lambda:=G_0+\lambda I_n.
\]

Independently, we generated \(U\in\mathbb R^{n\times q}\), with \(q=10\), by applying the same QR procedure to a new standard Gaussian matrix, and considered
\[
    T:=UB_2^q.
\]
The matrix \(G_0\) and subspace basis \(U\) were fixed across all values of \(\lambda\), using random seed \texttt{20260720}.

For each \(\lambda\), we estimated
\[
    w_{G_\lambda}(T),
    \qquad
    w_{G_\lambda^{-1}}(T),
    \qquad
    w(T),
\]
using \(10^5\) Monte Carlo samples. We also computed the normalized product ratio
\[
    \rho_\lambda
    :=
    \frac{
        w_{G_\lambda}(T)
        w_{G_\lambda^{-1}}(T)
    }{
        w(T)^2
    }.
\]

As \(\lambda\) increases, the primal width decreases from approximately \(8.26\) to \(0.86\), while the inverse-Fisher width increases from approximately \(2.93\) to \(11.09\). Thus the two widths move in opposite directions under this regularization path. At the same time, the product ratio decreases from approximately \(2.54\) toward equality:
\[
    \rho_0\approx2.536,
    \qquad
    \rho_{12}\approx1.003.
\]
The minimum value over the tested grid is
\[
    \min_\lambda\rho_\lambda\approx1.0026,
\]
consistent with the exact inequality
\[
    w_{G_\lambda}(T)
    w_{G_\lambda^{-1}}(T)
    \ge
    w(T)^2.
\]

The experiment illustrates width redistribution for one fixed matrix and one fixed subspace. It does not imply monotonicity of either width for arbitrary sets or arbitrary matrix paths; such behavior depends on the alignment of \(T\) with the eigenspaces of \(G_\lambda\).

\begin{figure}[tbp]
    \centering
    \includegraphics[width=\textwidth]
    {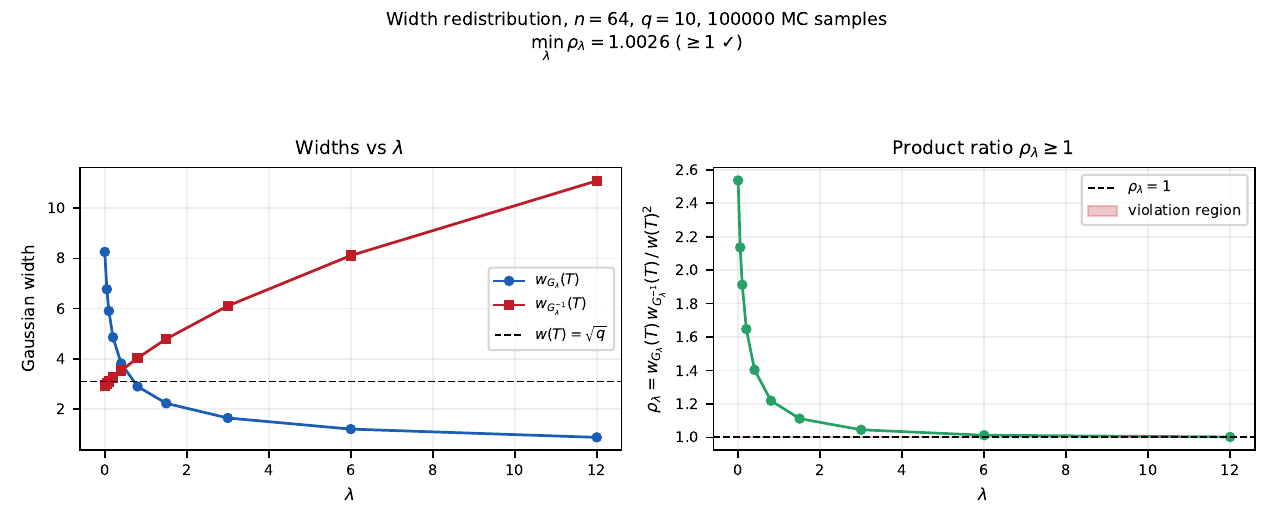}
    \caption{Primal--inverse width redistribution for
    \(G_\lambda=G_0+\lambda I_{64}\) on the subspace ball
    \(T=UB_2^{10}\). Widths are estimated using \(10^5\) Monte Carlo
    samples. Along this path, the primal width decreases and the
    inverse-Fisher width increases, while the normalized product
    \(\rho_\lambda\) approaches \(1\) from above. The figure is an
    illustration of the redistribution mechanism, not a numerical
    proof of the product inequality.}
    \label{fig:width-redistribution}
\end{figure}

\section*{Acknowledgments}

The author acknowledges the use of ChatGPT and Claude in the preparation of this manuscript. These tools were used to refine the language and organization of the draft, to brainstorm and explore proof strategies, and to assist in generating code for the numerical experiments. All mathematical arguments were independently checked, and all source code was reviewed and debugged by the author. The author takes full responsibility for the originality, correctness, and final content of the manuscript.

\bibliographystyle{plainnat}
\bibliography{references}

\appendix

\section{Verification of Fisher-regularity conditions}
\label{app:verification}

We verify the conditions of
Definition~\ref{def:fisher-regular} for the model classes appearing in Corollary~\ref{cor:verification}. Throughout, \(G\) denotes the Fisher matrix at the reference parameter \(\theta_0\), and all local Hessian bounds are required on the neighborhood
\[
    \{\theta:\|\theta-\theta_0\|_G\le\rho\}.
\]
We write
\[
    \zeta
    :=
    G^{-1/2}\nabla_\theta\ell_{\theta_0}(Z)
\]
for the whitened gradient.

\paragraph{(i) Logistic regression.}

Let \(Z=(X,Y)\), where \(Y\in\{0,1\}\), and consider
\[
    \ell_\theta(X,Y)
    =
    -Y\langle\theta,X\rangle
    +
    \log\bigl(1+e^{\langle\theta,X\rangle}\bigr).
\]
Writing \(\sigma(t)=(1+e^{-t})^{-1}\), we have
\[
    \nabla_\theta\ell_\theta(X,Y)
    =
    \bigl(\sigma(\langle\theta,X\rangle)-Y\bigr)X,
    \qquad
    \nabla_\theta^2\ell_\theta(X,Y)
    =
    \sigma'(\langle\theta,X\rangle)XX^\top.
\]

Under correct specification,
\[
    \mathbb E
    \nabla_\theta\ell_{\theta_0}(Z)
    =
    0,
    \qquad
    \operatorname{Cov}
    \bigl(
        \nabla_\theta\ell_{\theta_0}(Z)
    \bigr)
    =
    G,
\]
so Condition~\ref{fr:score} holds.

Since \(0\le\sigma'(t)\le1/4\), the \(G\)-Cauchy--Schwarz inequality gives
\[
\begin{aligned}
    \bigl|
        h^\top\nabla_\theta^2\ell_\theta(X,Y)h
    \bigr|
    &\le
    \frac14(h^\top X)^2 \\
    &\le
    \frac14
    \bigl(X^\top G^{-1}X\bigr)\|h\|_G^2.
\end{aligned}
\]
Thus Condition~\ref{fr:hessian} holds with
\[
    M(Z)
    :=
    \frac14X^\top G^{-1}X,
    \qquad
    \sigma_H
    :=
    \frac14
    \left[
        \mathbb E
        \bigl(X^\top G^{-1}X\bigr)^2
    \right]^{1/2},
\]
provided
\[
    \mathbb E
    \bigl(X^\top G^{-1}X\bigr)^2
    <
    \infty.
\]
In particular, this condition follows from the bounded-leverage assumption \(X^\top G^{-1}X\le\Lambda\) almost surely.

At \(\theta_0\),
\[
    \zeta
    =
    \bigl(
        \sigma(\langle\theta_0,X\rangle)-Y
    \bigr)G^{-1/2}X.
\]
Since
\(
    |\sigma(\langle\theta_0,X\rangle)-Y|\le1,
\)
\[
    \|\zeta\|_2^4
    \le
    \bigl(X^\top G^{-1}X\bigr)^2.
\]
Hence Condition~\ref{fr:nondegen} holds whenever
\[
    \mathbb E
    \bigl(X^\top G^{-1}X\bigr)^2
    \le
    \kappa^4d^2.
\]

\paragraph{(ii) Canonical-link generalized linear models.}

Consider a canonical exponential-family model with negative log-likelihood
\[
    \ell_\theta(X,Y)
    =
    A(\eta_\theta(X))
    -
    \langle Y,\eta_\theta(X)\rangle,
    \qquad
    \eta_\theta(X)=J(X)\theta+b(X).
\]
Then
\[
    \nabla_\theta\ell_\theta(X,Y)
    =
    J(X)^\top
    \bigl(
        \nabla A(\eta_\theta(X))-Y
    \bigr),
\]
and
\[
    \nabla_\theta^2\ell_\theta(X,Y)
    =
    J(X)^\top
    \nabla^2A(\eta_\theta(X))
    J(X).
\]

Assume the following:

\begin{enumerate}[label=\textup{(\alph*)}]

\item throughout \(\|\theta-\theta_0\|_G\le\rho\),
\[
    \nabla^2A(\eta_\theta(X))
    \preceq
    M_0I
    \qquad\text{almost surely};
\]

\item
\[
    \mathbb E
    \left\|
        G^{-1/2}J(X)^\top
    \right\|_{\mathrm{op}}^4
    <
    \infty;
\]

\item
\[
    \mathbb E
    \left\|
        G^{-1/2}J(X)^\top
        \bigl(
            \nabla A(\eta_{\theta_0}(X))-Y
        \bigr)
    \right\|_2^4
    \le
    \kappa^4d^2.
\]

\end{enumerate}

Under correct specification, the conditional gradient has mean zero and its covariance is the Fisher information matrix, so Condition~\ref{fr:score} holds. Moreover,
\[
\begin{aligned}
    h^\top\nabla_\theta^2\ell_\theta(X,Y)h
    &\le
    M_0\|J(X)h\|_2^2 \\
    &\le
    M_0
    \left\|
        G^{-1/2}J(X)^\top
    \right\|_{\mathrm{op}}^2
    \|h\|_G^2.
\end{aligned}
\]
Thus Condition~\ref{fr:hessian} holds with
\[
    M(Z)
    :=
    M_0
    \left\|
        G^{-1/2}J(X)^\top
    \right\|_{\mathrm{op}}^2.
\]
Finally,
\[
    \zeta
    =
    G^{-1/2}J(X)^\top
    \bigl(
        \nabla A(\eta_{\theta_0}(X))-Y
    \bigr),
\]
so assumption~\textup{(c)} is exactly
Condition~\ref{fr:nondegen}. For models with unbounded responses, these assumptions require an appropriate conditional moment or tail bound.

\paragraph{(iii) Gaussian linear regression.}

Let
\[
    Y=X^\top\theta_0+\varepsilon,
    \qquad
    \varepsilon\sim N(0,\sigma^2),
\]
where \(\varepsilon\) is independent of \(X\), and consider
\[
    \ell_\theta(X,Y)
    =
    \frac{1}{2\sigma^2}
    \bigl(Y-X^\top\theta\bigr)^2.
\]
Then
\[
    \nabla_\theta\ell_{\theta_0}(X,Y)
    =
    -\frac{\varepsilon}{\sigma^2}X,
    \qquad
    \nabla_\theta^2\ell_\theta(X,Y)
    =
    \frac{1}{\sigma^2}XX^\top.
\]
Furthermore,
\[
    G
    =
    \operatorname{Cov}
    \bigl(
        \nabla_\theta\ell_{\theta_0}(X,Y)
    \bigr)
    =
    \frac{1}{\sigma^2}\mathbb E[XX^\top],
\]
so Condition~\ref{fr:score} holds.

For every \(h\in\mathbb R^d\),
\[
\begin{aligned}
    h^\top\nabla_\theta^2\ell_\theta(X,Y)h
    &=
    \frac{1}{\sigma^2}(h^\top X)^2 \\
    &\le
    \frac{1}{\sigma^2}
    \bigl(X^\top G^{-1}X\bigr)\|h\|_G^2.
\end{aligned}
\]
Thus Condition~\ref{fr:hessian} holds with
\[
    M(Z)
    :=
    \frac{1}{\sigma^2}X^\top G^{-1}X,
    \qquad
    \sigma_H
    :=
    \frac{1}{\sigma^2}
    \left[
        \mathbb E
        \bigl(X^\top G^{-1}X\bigr)^2
    \right]^{1/2}.
\]

The whitened gradient is
\[
    \zeta
    =
    -\frac{\varepsilon}{\sigma^2}G^{-1/2}X.
\]
Using independence and
\(\mathbb E\varepsilon^4=3\sigma^4\),
\[
    \mathbb E\|\zeta\|_2^4
    =
    \frac{3}{\sigma^4}
    \mathbb E
    \bigl(X^\top G^{-1}X\bigr)^2.
\]
Hence Conditions~\ref{fr:hessian} and~\ref{fr:nondegen} both follow from
\[
    \mathbb E
    \bigl(X^\top G^{-1}X\bigr)^2
    <
    \infty,
\]
with \(\kappa\) chosen so that
\[
    \frac{3}{\sigma^4}
    \mathbb E
    \bigl(X^\top G^{-1}X\bigr)^2
    \le
    \kappa^4d^2.
\]

\section{Dual coordinates in exponential families}
\label{app:expfam}

Let
\[
    p_\eta(x)
    =
    \exp\bigl(
        \eta^\top t(x)-A(\eta)
    \bigr)
\]
be a regular minimal exponential family. Its mean parameter is
\[
    \mu=\nabla A(\eta),
\]
and the inverse relation is \(\eta=\nabla A^*(\mu)\), where \(A^*\) is the Legendre dual of \(A\).

Fix \(\eta_0\), and write
\[
    \mu_0:=\nabla A(\eta_0),
    \qquad
    G:=\nabla^2A(\eta_0)\succ0.
\]
Differentiating
\(\nabla A^*(\nabla A(\eta))=\eta\) at \(\eta_0\) gives
\[
    \nabla^2A^*(\mu_0)\nabla^2A(\eta_0)=I_d,
\]
and therefore
\[
    \nabla^2A^*(\mu_0)
    =
    \bigl[\nabla^2A(\eta_0)\bigr]^{-1}
    =
    G^{-1}.
\]
Thus the Fisher metric in natural coordinates and its inverse in mean coordinates arise as the Hessians of a Legendre-dual pair. This provides a canonical dual-coordinate interpretation of the two metric deformations used in the main text.

The interpretation does not make same-coordinate comparisons invariant. For example, under the one-dimensional rescaling
\[
    \eta'=c\eta,
    \qquad c\neq1,
\]
the Fisher information transforms as \(G'=c^{-2}G\), while a fixed numerical interval \(T=[-\delta,\delta]\) in the \(\eta'\)-chart does not represent the same tangent perturbations as the interval with the same endpoints in the \(\eta\)-chart.

The preceding identities concern natural and mean coordinates transformed according to their dual laws. Same-coordinate comparisons in the main text are interpreted as in Remark~\ref{rem:fixed_chart}.

\end{document}